\documentclass{alt2019}
\PassOptionsToPackage{numbers}{natbib}

\usepackage{color,soul}

\usepackage[utf8]{inputenc} 
\usepackage[T1]{fontenc}    
\usepackage{hyperref}       
\usepackage{url}            
\usepackage{booktabs}       
\usepackage{amsfonts}       
\usepackage{nicefrac}       
\usepackage{microtype}      
\usepackage{xspace}
\usepackage{amsmath}

\usepackage{mathtools}  
\usepackage{amsmath}
\usepackage{amssymb}
\usepackage{tabulary}
\usepackage{booktabs}

\newtheorem{assumption}{Assumption}
\usepackage{algorithm}
\usepackage{algpseudocode}
\usepackage{algpascal}
\usepackage{comment}
\usepackage{times}
\usepackage{selectp}

\newcommand{\Sw}{\mathcal{S}}

\newcommand{\real}{\mathbb{R}}

\newcommand{\OO}{\mathcal{O}}

\newcommand{\II}[1]{\mathbb{I}_{\left\{#1\right\}}}
\newcommand{\PP}[1]{\mathbb{P}\left[#1\right]}

\newcommand{\EE}[1]{\mathbb{E}\left[#1\right]}

\newcommand{\EXP}{\mathbb{E}}

\newcommand{\PPcc}[2]{\mathbb{P}\left[\left.#1\right|#2\right]}

\newcommand{\EEcc}[2]{\mathbb{E}\left[\left.#1\right|#2\right]}

\def\argmax{\mathop{\mbox{ arg\,max}}}
\newcommand{\ra}{\rightarrow}

\newcommand{\iprod}[2]{\left\langle#1,#2\right\rangle}

\newcommand{\ev}[1]{\left\{#1\right\}}
\newcommand{\pa}[1]{\left(#1\right)}
\newcommand{\bpa}[1]{\bigl(#1\bigr)}

\newcommand{\wh}{\widehat}
\newcommand{\wt}{\widetilde}

\newcommand{\transpose}{^\mathsf{\scriptscriptstyle T}}

\usepackage{todonotes}
\definecolor{PalePurp}{rgb}{0.66,0.57,0.66}

\newcommand{\qed}{\hfill\BlackBox\\[2mm]}

\newcommand{\ucbvdouble}{\textsc{UCB(}$V_0$\textsc{)-Double }}

\newcommand{\krgraph}{$G_{n, P^{[k]}}$}
\newcommand{\ducb}{$d$-UCB\xspace}

\newcommand{\ducbv}{$d$-UCB$(V_0)$\xspace}
\newcommand{\lucbv}{Local UCB$(V_0)$\xspace}
\newcommand{\klucb}{kl-UCB\xspace}
\newcommand{\irg}{$G(n, \kappa)$\xspace}
\newcommand{\chlu}{$G(n, w)$\xspace}
\newcommand{\sbm}{$G(n, \alpha, K)$\xspace}
\newcommand{\chunglu}{Chung--Lu\xspace}
\newcommand{\GW}{W}

\newcommand{\poi}{\text{Poi}}

\newcommand{\ber}{\text{Ber}}
\newcommand{\hmu}{\wh{\mu}}
\newcommand{\hmusub}{\wh{u}}
\newcommand{\hmusup}{\wh{v}}
\newcommand{\musup}{v}
\newcommand{\musub}{u}
\def\qed{\hfill$\Box$\medskip}

\begin{document}

\title{Learning to maximize global influence from local
	observations \thanks{A part of the material of the paper appeared in the   Proceedings of ALT 2019}}

\coltauthor{\Name{G\'abor Lugosi} \Email{gabor.lugosi@gmail.com}\\
	\addr ICREA \& Universitat Pompeu Fabra\\
	Barcelona, Spain
	\AND
	\Name{Gergely Neu} \Email{gergely.neu@gmail.com}\\
	\addr Universitat Pompeu Fabra\\
	Barcelona, Spain
	\AND
	\Name{Julia Olkhovskaya} \Email{julia.olkhovskaya@gmail.com}\\
	\addr Universitat Pompeu Fabra\\
	Barcelona, Spain
}

\editor{}
\maketitle

\begin{abstract}
		We study a family online influence maximization
                problems where in a sequence of rounds $t=1,\ldots,T$, a decision maker
                selects one from a large number of agents with the
                goal of maximizing influence. 
                Upon choosing an agent, the 
		decision maker shares a piece of information with the agent, which information then spreads in an unobserved 
network over which the agents communicate. The goal of the decision maker is to select the sequence of agents in a way 
that the total number of influenced nodes in the network. In this work, we consider a scenario where the networks are 
generated independently for each $t$ according to some fixed but unknown distribution, so that the set of
                influenced nodes corresponds to the connected component of the
                random graph containing the vertex corresponding to
                the selected agent. Furthermore, we assume that the decision maker only has access to limited 
feedback: instead of making the unrealistic assumption that the entire network is observable, we suppose that the 
available feedback is generated based on a small neighborhood of the selected vertex. 
Our  results
		show that such partial local  observations 
		can be sufficient for maximizing 
		global influence. We model the underlying random graph
                as a sparse inhomogeneous  Erd\H{o}s--R\'enyi graph, and study three specific families of random graph 
models in detail: 
		stochastic block models, Chung--Lu models and
                Kronecker random graphs. We show that in these cases
                one may learn to maximize influence by merely
                observing the degree of the selected vertex in the
                generated random graph. We propose
		 sequential learning  algorithms that aim at maximizing influence,
		and provide their theoretical analysis in both the subcritical 
		and supercritical regimes of all considered models. 
	\end{abstract}


	\section{Introduction}
Finding influential nodes in  networks has a long history of study. The problem 
has been cast in a variety of different
ways according to the notion of influence and the information
available to a decision maker. We refer the reader to
\citet*{KeKlTa03,ChenSIM10,ChLaCa13,VaLaSc15,CaVa16,WeKvVaVa17,WaCh17, khim2019adversarial,perrault2020budgeted} and the references therein 
for recent progress in various directions. 

The most studied influence maximization setup is an offline discrete optimization problem of 
finding the set of the most influential nodes in a network. This setup assumes
that the probability of influencing is known, or at least 
data is available that allows one to estimate these probabilities. However, such information is 
often not available or is difficult to obtain. Also, the network over which information 
spreads is
rarely fixed.
To avoid such assumptions, we introduce a novel model of influence maximization in a sequential setup, 
where the underlying network changes every time and the learner has only partial information about the set of influenced nodes.

Specifically, we define and explore a
sequential decision-making model in which the goal of 
a decision maker is to find one among a set of $n$ agents
with maximal (expected) influence. We parametrize the information spreading mechanism by a symmetric $n\times n$ matrix 
$P$, whose entries $p_{i,j} \in [0,1]$ express ``affinity'' or ``probability of communication'' between agents $i$ and 
$j$.
We assume that $p_{i,i} = 0$ for all $i\in [n]$.
The matrix $P$ defines an inhomogeneous random graph $G$ in a natural way: an (undirected) edge is present between nodes $i <j$ with probability $p_{i,j}$ and all edges are independent. When two nodes are connected by an edge, information flows between the corresponding agents. Hence, a piece of information placed at a node $i$ spreads to the nodes of the entire connected component of $i$ in $G$.

In the sequential decision-making process we study, an independent random graph is formed at each time
instance $t=1,\ldots,T$ on the vertex set $[n]$. The random graph formed at time $t$ is denoted by $G_t$.
Hence, $G_1,\ldots,G_T$ is an independent, identically distributed sequence of random graphs on the
vertex set $[n]$, whose distribution is determined by the matrix $P$.
If the decision maker selects a node $a\in [n]$ at time $t$, then the information placed at the node spreads to every
node of the connected component of $a$ in the graph $G_t$.
The goal of the
decision maker is to spread information as much as possible, that is, to reach as many agents as possible.
The \emph{reward} of the decision maker at time $t$ is the number of nodes in the connected component containing the selected node in $G_t$. 

In this paper, we study a setting where the decision maker has no prior knowledge of the distribution $P$, so 
she has to learn about this distribution on the fly, while simultaneously attempting to maximize the total reward. This gives rise to a 
dilemma of \emph{exploration versus exploitation}, commonly studied within the framework of \emph{multi-armed bandit} problems 
(for a survey, see \citealp{bubeck12survey} or \citealp{LSz20}). Indeed, if the decision maker could observe the size 
of the set of all influenced nodes in every round, the sequential influence 
maximization problem outlined above could be naturally formulated as a \emph{stochastic multi-armed bandit} problem 
\citep*{LR85,auer2002finite}. However, this direct approach has multiple drawbacks. First of all, in many applications, 
the number $n$ 
of nodes is so large that one cannot even hope to maintain individual statistics about each of them, let alone expect any algorithm to 
identify the most influential node in reasonable time. More importantly, in most cases of interest, tracking down the 
set of 
\emph{all} influenced agents may be difficult or downright impossible due to privacy and computational considerations. This motivates the 
study of a more restrictive setting where the decision maker has to manage with only partial observations of the set of influenced nodes.

We address this latter challenge by considering a more realistic observation model, where after selecting an agent $A_t$  
to be influenced, the learner only observes a local neighbourhood of  $A_t$ in the realized random graph $G_t$ ,  or even only the number of immediate neighbours of $A_t$ (i.e., the 
degree of vertex $A_t$ in $G_t$). This model raises the following question: is it possible to maximize global influence while 
only having access to such local measurements? Our key technical result is answering this question in the positive for some broadly studied 
random graph models. 

The rest of the paper is structured as follows.
In Section \ref{sec:prelim} we formalize the sequential influence maximization problem.
In Section \ref{sec:randomgraphmodel} a general model of inhomogeneous random graphs is described and the crucial
notions of sub-, and super-criticality are formally introduced.
Section~\ref{sec:censored} is dedicated to the general case when the underlying random graph is an arbitrary
inhomogenous random graph and the learner only knows whether it is in the subcritical or supercritical regime.
We show that in both cases online influence maximization is possible by only observing a small ``local'' neighborhood
of the selected node. We provide two separate algorithms and regret bounds for the subcritical and supercritical cases, respectively.
In Section~\ref{sec:degree}, we consider the situation when the learner has even less information about the underlying
random graph. In particular, we assume that the learner only observes the degree of the selected node
in the realized random graph. We study three well-known special cases of inhomogeneous random graphs that
are commonly used to model large social networks, namely stochastic block models, the Chung--Lu model, and
Kronecker random graphs. We prove that in these 
three random graph models, degree observations are sufficient to maximize global influence both in the subcritical
and supercritical regimes.
In Section~\ref{sec:conc} we provide some discussion and comparison to the previous work. In sections \ref{sec:branching}, \ref{component_concentration}, \ref{sec:irg_supercrit}, and \ref{sec:degree} we present all proofs.

	\subsection{Problem setup}
\label{sec:prelim}
We now describe our problem and model assumptions formally. We
consider the problem of sequential influence maximization  on the set of nodes $V=[n]$, formalized as a repeated
interaction scheme between a learner and its environment. We
assume that node $i$ influences node $j$ with (unknown) probability
$p_{i,j} (=p_{j,i})$. At each
iteration, a new graph $G_t$ is generated on the vertex set $V$ by independent
draws of the edges such that edge $(i,j)$ is present  with probability
$p_{i,j}$ and all edges are independent.
The set of nodes influenced by the chosen node $A_t$ is the connected component of $G_t$ that 
contains $A_t$. $C_{i,t}$ denotes the connected component containing vertex $i$:
\[
C_{i,t} = \ev{v\in V: \mbox{$v$ is connected to $i$ by a path in $G_t$}}~.
\]
The feedback that the decision maker receives after choosing a node is
some ``local'' information around the chosen vertex $A_t$ in $G_t$.
We consider several feedback models. In the simplest case,  the
feedback is the degree of vertex $A_t$ in $G_t$. In another model,
the information might consist of the vertices found after a few steps
of depth-first exploration of $G_t$ started from vertex $A_t$.  
In a general framework, we may define a ``local  neighborhood" of
$A_t$, denoted by $\widehat{C}_{A_t,t}$, where $\widehat{C}_{A_t, t}
\subset C_{A_t,t}$.
For each model considered below, we specify later what exactly
$\widehat{C}_{A_t,t}$ is.
In the general setup, the following 
steps are repeated for each round $t=1,2,\dots$:
\begin{enumerate}
	\item the learner picks a vertex $A_t \in V$,
	\item the environment generates a random graph $G_t$,
	\item the learner observes the local neighborhood  $\widehat{C}_{A_t,t}$,
	\item the learner earns the reward $r_{t,A_t} = |C_{A_t,t}|$.
\end{enumerate}
We stress that the learner does \emph{not} observe the reward, only the local neighborhood $\widehat{C}_{A_t,t}$.  
Define $c_i$ as the expected size of the connected component associated with the node $i$:  $c_i = \EE{|C_{i,1}|}$. Ideally, one would like to minimize the \emph{expected regret} defined as
\begin{equation}\label{eq:regret}
R_T = \EE{ \sum_{t=1}^T \pa{ \max_{i\in V} c_i - c_{A_t}}}.
\end{equation}
 
Since we are 
interested in settings where the total number of nodes $n$ is very large, even with a fully 
known random graph model, finding the optimal node maximizing $c_i$ is infeasible both computationally and 
statistically. Such intractability issues have lead to 
alternative definitions of the regret such as the \emph{approximation regret} \cite*{KKL09,CWY13,SG09} or the \emph{quantile 
	regret} \cite*{CFH09,CV10,LS14,KvE15}. 

In the present paper, we consider the $\alpha$-quantile regret as our performance measure, which, instead of measuring the learner's 
performance against the single best decision, uses a near-optimal action as a baseline. For a more technical definition, let $i_1, i_2, 
\dots, i_n$ be an ordering of the nodes satisfying  $c_{i_1} \le c_{i_2} \le \dots \le c_{i_n}$, and  denote the $\alpha$-quantile over 
the mean rewards as $c^*_\alpha = c_{i_{\lceil(1-\alpha) n\rceil}}$. Then, defining the set 
$V^*_{\alpha} = \{i_{\lceil(1-\alpha) n\rceil}, \dots, i_{n}\}$ as the set of $\alpha$-near-optimal nodes, we define the 
$\alpha$-quantile regret as

\begin{equation}\label{eq:qregret}
R^{\alpha}_T =  \EE{\sum_{t=1}^T \pa{\min_{i\in V^*_{\alpha}} c_i - c_{A_t}}} = \EE{\sum_{t=1}^T \pa{c^*_{\alpha} - c_{A_t}}}~. 
\end{equation}

	\subsection{Inhomogeneous Erd\H{o}s--R\'enyi random graphs}
\label{sec:randomgraphmodel}

Next we discuss the random graph models considered in this paper. 
All belong to the \emph{inhomogeneous Erd\H{o}s--R\'enyi model}, that is,
edges are present independently of each other, with possibly different
probabilities. Moreover, the
graphs we consider are \emph{sparse} graphs, that is, the average degree is bounded.
We will formulate our random graph model following the work of \cite*{Bollobas:2007:PTI:1276871.1276872}, whose 
framework is particularly useful for handling large values of $n$. To this end, let $\kappa$ be a bounded symmetric
non-negative measurable function on $[0,1] \times [0,1]$. 
Each edge $(i,j)$ for $1\le i< j\le n$ is present with probability $p_{i,j}=\min(\kappa(i/n, j/n)/n,1)$, independently 
of all other edges. When $n$ is fixed, we will often use the notation $A_{i,j} = \kappa(i/n, j/n)$ so that $p_{i,j} = 
\min(A_{i,j}/n,1)$.
We are interested in random graphs where the average degree is $O(1)$ (as $n\to \infty$).
This assumption makes the problem both more realistic and challenging: 
denser graphs are connected with high probability, making  the problem
essentially vacuous. A random graph drawn from the above distribution 
is denoted by \irg. This model is sometimes called the \emph{binomial random graph} and was first considered by 
\cite{kovalenko1971theory}.

We consider two fundamentally different regimes of the parameters \irg: the \emph{subcritical} case 
in which the size of the largest connected component is sublinear in $n$ (with high probability), and the \emph{supercritical} case where the largest connected 
component is at least of size $c n$ for some constant $c>0$, with high probability. 
(We say that an event holds \emph{with high probability} if its probability converges to one as $n\to \infty$.)
Such a connected component of linear size is called a \emph{giant component}.
These regimes can be formally characterized with the help of the integral operator  $T_{\kappa}$, defined by 
\[
  \left(T_{\kappa} f\right)(x) = \int_{(0,1]}\kappa(x,y)f(y)d\mu(y)~,
\] 
for any measurable bounded function $f$, where $\mu$ is the Lebesgue measure.  We call $\kappa$  subcritical if  
$\|T_{\kappa}\|_2 < 1$ and supercritical if 
$\|T_{\kappa}\|_2 > 1$. We use the same expressions for a random graph \irg. It follows from \citet*[Theorem 3.1]{Bollobas:2007:PTI:1276871.1276872} that, with high probability,
\irg has a giant component if it is supercritical, while the number of vertices in the
largest component is $o(n)$ with high probability if it is subcritical.

	\section{Observations of censored component size}\label{IRG}
\label{sec:censored}

First we study a natural feedback model in which the decision maker, unable to explore
the entire connected component $C_{i,t}$ of the influenced node $i$ in $G_t$, resorts to exploring the
connected component up to a certain (small) number of nodes.
More precisely, we define feedback as the result of counting the number of nodes in $C_{i,t}$ by (say, depth-first search)
exploration of the connected component, which stops after revealing $K$ nodes, or before, if $|C_{i,t}| < K$.
Here $K$ is a fixed positive integer, independent of the number of nodes $n$.

The main results of this section show that this type of feedback is sufficient for sequential influence maximization.
However, the subcritical and supercritical cases need to be treated separately as they are quite different.
In the subcritical case, the expected size of the connected component of any vertex is of constant order
while in the supercritical case there exist vertices whose connected component is linear in $n$. This also
means that the rewards -- and therefore the per-round regrets -- are of different order of magnitude (as a function of $n$)
in the subcritical and supercritical cases. For simplicity, we assume that the decision maker knows in advance whether the function $\kappa$
defining the inhomogeneous random graph is subcritical or supercritical, as we propose different algorithms
for both cases. We believe that this is a mild assumption, since in typical applications it is possible to set 
the two settings apart based on prior data. We also assume that $\|T_{\kappa}\|_2 \neq 1$,
that is, the random graph is not exactly critical. 

\subsection{Subcritical case}

First we study the subcritical case, that is, we assume that $||T_{\kappa}||_2 < 1$.
In this case the proposed influence-maximization algorithm uses the censored size of the
connected component of the selected node. That is, for a node $i\in [n]$, we define $\musub_{i,t}(K)$ as the result of counting the number of nodes in $C_{i,t}$ by exploration of the connected component, which stops after revealing $K$ nodes or before, if $|C_{i,t}| < K$. Hence, the feedback is $\musub_{ i, t}(K) =\min\pa{|C_{i,t}|, K} $.

A key ingredient in our analysis in the subcritical case is an estimate for the lower tail of the size of the connected component containing a fixed vertex. We state it in the following lemma: 

\begin{lemma}\label{subcritical}
	For any subcritical $\kappa$, there exist positive constants $\lambda(\kappa), g(\kappa)$ and $n_0(\kappa)$, such that  for any  $n\ge n_0$, for any node $i$ in \irg, the size of the connected component $C_i$ of a vertex $i$ satisfies
	\begin{equation} \label{ineq_main}
		\PP{|C_i| > u}\le e^{-\lambda(\kappa) u}g(\kappa)~.
	\end{equation}
\end{lemma}
 Unfortunately, there  is no closed-form expression for the dependence $\lambda(\kappa)$ and $g(\kappa)$ on $\kappa$. 
The idea of the proof of this lemma relies on the proof of  Theorem 12.5 in \citet*{Bollobas:2007:PTI:1276871.1276872}. 
To obtain this result, we show that  the size of the connected component in \irg is stochastically 
dominated by the total progeny of the multitype Poisson branching process with carefully chosen parameters.  We 
introduce branching processes in Section~\ref{sec:branching} and prove Lemma~\ref{subcritical} in 
Section~\ref{component_concentration}. 
 
 Now we are ready to define an estimate of $c_i=\EXP|C_i|$ in the sequential decision game.
 For a fixed a constant $K$, we define  the estimate $\widehat{u}_{i,t}(K) = (1/t) \sum_{s=1}^{t} u_{i,s}(K)\II{A_s = i}$.  Using the concentration inequality (\ref{ineq_main}), with the choice of the threshold parameter $K = \frac{\log(T)}{\lambda}$ with $\lambda > \lambda(\kappa)$, we get  that  the bias of $\widehat{u}_{i,t}(K)$ is at most $\frac{g(\kappa)}{T}$. We state this result more formally in Lemma~\ref{prop:irg_subcr}. The censored observations are bounded, since  $ \musub_{i,t}(K) \in [1, K]$. We use those observations as rewards in our bandit problem and we  feed them to an instance of the UCB algorithm \citep*{auer2002bandit}.
 We call the resulting algorithm \lucbv, defined in Algorithm \ref{alg:thealg1sub} below.


A minor challenge is that, since we are interested in very large values of $n$, it is infeasible to use \emph{all} 
nodes as separate actions in our bandit algorithm. To address this challenge, we propose to \emph{subsample} a set of representative nodes 
for UCB to play on. The size of the subsampled nodes depends on the quantile $\alpha$ targeted in the regret 
definition~\eqref{eq:qregret} and the time horizon $T$. Our  algorithm uniformly samples
a subset $V_0$ of size
\begin{equation}\label{eq:v0}
|V_0| =  \left \lceil \frac{\log T}{\log (1/(1- \alpha))} \right\rceil
\end{equation}
and plays \lucbv for the corresponding regime on the resulting set. Note that the size of
$V_0$ is chosen such that the probability that $V_0$ does not contain any of the $\alpha n$  notes
with the largest values of $c_i$ is at most $1/T$.

To simplify the presentation, we introduce some more notation. Analogously to the $\alpha$-optimal reward 
$c^*_\alpha$, we define the $\alpha$-optimal censored component size $\musub_{*,\alpha}(K) = \min_{i\in V^*_{\alpha}} \musub_i (K)$ and we define the corresponding 
gap parameters $\Delta_{\alpha,i} = \pa{ c^*_\alpha - c_i }_+$, 
$\delta^{sub}_{\alpha,i}(K)  = \pa{\musub_{*,\alpha}(K) - \musub_i(K)  }_+$ and $\Delta_{\alpha,\max} = \max_i \Delta_{\alpha,i}$.
$N_{i,t}= \sum_{s=1}^t\II{A_s=i}$ denotes the number of times node $i$ is selected up to time $t$.

\begin{algorithm}[h]
	\caption{\lucbv for subcritical \irg.}
	\label{alg:thealg1sub}
	\textbf{Parameters:} A set of nodes $V_0 \subseteq V$, $K > 0$.\\
	\textbf{Initialization:} Select each node in $V_0$ once. 
	For each $i\in V_0$, set $N_{i, |V_0| }=1$ and ${\hmusub}_{i,|V_0|}= u_{i, i}(K)$.
	\\
	\textbf{For} $t = |V_0|, \dots T$, \textbf{repeat}
	\begin{enumerate}
		\item Select any node $A_{t+1} \in \arg\max_{i} {\hmusub}_{i,t}(K) + K\sqrt{\frac{\log t}{N_{i,t}}}$. 
		\item Observe $u_{A_{t+1},t+1}(K)$ , update ${\hmusub}_{i, t+1}$ and 
		$N_{i, t+1}$ for all $i\in [n]$.
	\end{enumerate}
\end{algorithm}

For the subcritical case, \lucbv  has the following performance guarantee:

\begin{theorem}[Subcritical inhomogeneous random graph]
	\label{subcrit_reg}
	Assume that $\kappa$ is subcritical. Let $V_0$ be a uniform subsample of $V$ with size given in (\ref{eq:v0}) and define the event $\mathcal{E} = \ev{V_0 \cap 
		V^*_\alpha \neq \emptyset}$.
	Then for  any \irg with $n>n_0(\kappa)$ and any $K$,  the expected $\alpha$-quantile regret of \lucbv satisfies
	\begin{align*}
	R^{\alpha}_T \le   \Delta_{\alpha,\max} +  \EEcc{ \sum_{i\in V_0}  \Delta_{\alpha,i}  \pa{\frac{4 K^2 \log T}{(\delta^{sub}_{\alpha,i}(K))^2} + 8 }}{\mathcal{E}},
	\end{align*}
	where the expectation is taken over the random choice of $V_0$. Furthermore, if $\kappa$ is such that 
$\lambda(\kappa) > \lambda$, $g(\kappa) < g$, then, taking $K = \frac{\log T}{\lambda}$, we have
	\begin{align*}
	R^{\alpha}_T   
	\le \frac{4 \log T}{\lambda}\sqrt{\frac{T}{\log(1/(1-\alpha))}  } + 8 \Delta_{\alpha,\max} \left \lceil \frac{\log T}{\log (1/(1- \alpha))} \right\rceil  + 2g.
	\end{align*}
\end{theorem}

We prove Theorem~\ref{subcrit_reg} in Section~\ref{component_concentration}.
Observe that one may choose the value of $K$ as a constant, regardless of the number $n$ of the nodes.
This means that the feedback information is truly ``local'' in the sense that only a constant number
of vertices of the connected component of the selected node need to be explored. How large $K$ needs
to be depends on the parameter $\lambda$.
An undesirable feature of \lucbv is that the learner needs to know the parameter $\lambda$ that depends on the unknown function $\kappa$. To resolve this problem we propose a version of a "doubling trick" (see, e.g., Section 2.3 \cite{CBL06}).

While in our problem it is not possible to control the range of $\lambda(\kappa)$ explicitly, we still can control  the frequency with which $|C_i|$ is censored by choosing the range of $K$.  In order to do this, we propose a variation of \lucbv, such that we split time $T$  into episodes $q = 1,  2\dots$ in the following way.  At the beginning of each episode $q$, the learner starts a new instance of   \lucbv with a threshold parameter $K_q = 2^q \log T$ and starts a new time counter $t_q$. Then, at each time step of the current episode,  the learner  computes the empirical probability $\widehat{p}_q =\frac{1}{t_q}\sum_{\tau = t_{q-1} + 1}^{t_q} \II{|C_{A_{\tau}, \tau}| > K_q  } $, that is updated each time when the size of connected component of the chosen node exceeds $K_q$.  Once  $\wh p_q $ gets larger than $\frac{1}{T} + \sqrt{\frac{\ln T}{2(t_q +1)}}$, the episode $q$ finishes and the next episode begins.  In this way, the length of each episode and the total number of episodes $Q_{\max}$ are random. We call this algorithm \ucbvdouble, and show that it has the following performance guarantee:

\begin{algorithm}[h]
	\caption{\ucbvdouble for subcritical \irg.}
	\label{alg:thealg1double}
	\textbf{Parameters:} A set of nodes $V_0 \subseteq V$, $T>0$.\\
	\textbf{Initialization:} $K_0 = \log T$, $t = 1$, $q = 0$, $t_q = 0$, $\widehat{p}_q = 0$.\\
	\textbf{While} $t \le T$, \textbf{repeat:}
	\begin{itemize}
	 \item Select each node in $V_0$ once. For each $i\in V_0$, set $N_{i, t }=1$, ${\hmusub}_{i,t}= u_{i, t}(K_q)$ and  $\widehat{p}_q = \frac{1}{|V_0|}\sum_{\tau = t_{q-1}+1}^{t_{q-1}+1 + |V_0|} \II{|C_{A_{\tau},\tau}| > K_q}$.
	 
	 \textbf{While} $\widehat{p}_q \le \frac{1}{T} + \sqrt{\frac{\ln T}{2 (t_q+1)}}$, \textbf{repeat:}
	  \begin{enumerate}
		 \item Select any node $A_{t_q+1} \in \arg\max_{i} {\hmusub}_{i,t_q}(K_q) + K_q\sqrt{\frac{\log T}{N_{i,t_q}}}$. 
		\item Observe $u_{A_{t_q+1},t+1}$,
		\item Update   ${\hmusub}_{A_{t_q+1}, t_q+1} = \frac{1}{t_q}\sum_{\tau = t_{q-1} + 1}^{t_q} u_{i, \tau}(K_q)\II{ A_{\tau} = A_{t_q+1}  } $, 
		$N_{A_{t_q+1}, t_q+1} = N_{A_{t_q}, t_q} + 1$, $\widehat{p}_q =\frac{1}{t_q}\sum_{\tau = t_{q-1} + 1}^{t_q} \II{|C_{A_{\tau}, \tau}| > K_q  }  $,
		\item Update   $t_q = t_q +1$ and  $t = t+1$.
	\end{enumerate}
	\item Set  $t_{q+1} = 0$, $\widehat{p}_{q+1}  = 0$, $K_{q+1} = 2 K_{q} $ and $q = q+1$.
	\end{itemize}
\end{algorithm}

\begin{theorem}
	\label{subcrit_reg_doubling}
	Assume that $\kappa$ is subcritical and $n>n_0(\kappa)$. Let $V_0$ be a uniform subsample of $V$ with size given in (\ref{eq:v0}) and define the event $\mathcal{E} = \ev{V_0 \cap 
		V^*_\alpha \neq \emptyset}$. 
	Then for  \irg with $n>n_0(\kappa)$, the expected $\alpha$-quantile regret of \ucbvdouble satisfies
	\begin{align*}
			R^{\alpha}_T & 
		\le \Delta_{\alpha, \max}  +  \frac{64}{3} \EEcc{ \sum_{i\in V_0}  \Delta_{\alpha,i}  \pa{
				\frac{ \log^3 T}{(\lambda(\kappa) \cdot \min_{q\in[Q_{\max}]}\{\delta^{sub}_{\alpha,i}(2^q \log T)\})^2} + 8 }  }{\mathcal{E}},
	\end{align*}
	where the expectation is taken over the random choice of $V_0$, and 
	\begin{align*}
		R^{\alpha}_T = \OO \pa{ \frac{\sqrt{ T (\log(1/\lambda(\kappa)) + 1)   }}{\lambda(\kappa)\sqrt{ \ln(1/(1-\alpha)) }} \log^2T  }.
	\end{align*}
\end{theorem}

The proof of Theorem~\ref{subcrit_reg_doubling} may be found in Section~\ref{component_concentration}.
Note that both Theorems~\ref{subcrit_reg} and~\ref{subcrit_reg_doubling} present two types of regret bounds. The 
first set of these bounds are polylogarithmic\footnote{Upon first 
	glance, the bound of Theorem~\ref{subcrit_reg} may appear to be logarithmic, however, notice that the sum involved in the bound has $\Theta(\log T)$ elements, thus technically resulting in a bound of order $\log^2 T$.} in the time horizon $T$, but show strong dependence on the 
parameters of the distribution of the graphs $G_t$. Such bounds are usually called \emph{instance-dependent}, as they 
are typically interesting in the regime where $T$ grows large and the problem parameters are fixed independently of 
$T$. However, these bounds become vacuous for smaller values of $T$ as the gap parameters 
$\delta^{sup}_{\alpha,i}(\cdot)$ and $\delta^{sub}_{\alpha,i}(\cdot)$ approach zero. 
This issue is addressed by our second set of guarantees, which offer a bounds of 
$\wt{O}\bpa{\sqrt{|U| T}}$ for some set $U\subseteq V$ that holds simultaneously for all problem instances without becoming 
vacuous in any regime. Such bounds are commonly called \emph{worst-case}, and they are often more valuable when 
optimizing performance over a fixed horizon $T$.

A notable feature of our bounds is that they show no explicit dependence on the number of nodes $n$. This is enabled by our notion of 
$\alpha$-quantile regret, which allows us to work with a small subset of the total nodes as our action set. 
Instead of $n$, our bounds depend on the size of some suitably chosen set of nodes $U$, which is of the order $\mathop{\mbox{polylog}}T / 
\log(1/(1-\alpha))$. Notice that this gives rise to a subtle tradeoff: choosing smaller values of $\alpha$ inflates the 
regret bounds, but, in exchange, makes the baseline of the regret definition stronger (thus strengthening the regret 
notion itself).

\subsection{Supercritical case}

Next we address the supercritical case, that is, when $||T_{\kappa}||_2  >1$.
Here the proposed algorithm uses $\musup_{i, t}(K)$ defined as the indicator whether   $|C_i|$ is larger than $K$, that is, $\musup_{i,t}(K) = \II{|C_i(G_t)| > K}$. Since the observation is an indicator function,   $ \musup_{i,t}(K) \in \{0,1\}$. Similarly to the subcritical case, we propose a variant of UCB algorithm, \lucbv, played over a random subsample of nodes of size defined in (\ref{eq:v0}).
We define $\musup_i (K) = \EE{\musup_{i,t}(K)} $ and $\musup_*(K)  = \max_i \musup_i (K) $. 
Analogously to the notation introduced for the subcritical regime, we denote  $\musup_{*,\alpha}(K) = \min_{i\in V^*_{\alpha}} \musup_i (K)$ and $\delta^{sup}_{\alpha,i}(K)  = \pa{\musup_i(K) - \musup_{*,\alpha}(K) }_+$.

In the supercritical case, the learner receives $\musup_{i,t}(K)$ as a reward and we design a bandit algorithm based on this form of indicator observations. Note, that  $\musup_{i,t}(K)$ is a Bernoulli random variable with  parameter $\PP{|C_i(G_t)| >K}$.
The following algorithm is a variant of the UCB algorithm of \citet*{auer2002bandit}.
Just like before, $N_{i,t}$ denotes the number of times node $i$ is selected up to time $t$ by the algorithm.
\begin{algorithm}[H]
	\caption{\lucbv for supercritical \irg.}
	\label{alg:thealg1sup}
	\textbf{Parameters:} A set of nodes $V_0 \subseteq V$, $k(n)$.\\
	\textbf{Initialization:} Select each node in $V_0$ once. 
	For each $i\in V_0$, set $N_{i, |V_0|}=1$ and ${\hmusup}_{i, |V_0|}(k(n))=\musup_{i, i}(k(n))$.
	\\
	\textbf{For} $t = |V_0|, \dots T$, \textbf{repeat}
	\begin{enumerate}
		\item Select any node $A_{t+1} \in \arg\max_{i} {\hmusup}_{i,t}(k(n)) + \sqrt{\frac{\log t}{N_{i,t}}}$. 
		\item Observe the feedback $\musup_{i, t}(k(n))$ , update
                  $\hmusup_{i, t+1}(k(n))$ and 
		$N_{i, t+1}$ for all $i\in [n]$.
	\end{enumerate}
\end{algorithm}

 \lucbv for supercritical \irg satisfies the following regret bound:
\begin{theorem}
	\label{supercrit_reg_unknown_l}
 Let $V_0$ be a uniform subsample of $V$ with size given in (\ref{eq:v0}) and define the event $\mathcal{E} = \ev{V_0 \cap 
		V^*_\alpha \neq \emptyset}$. For any \irg with supercritical $\kappa$ and $n>n_0(\kappa)$, 
	for  any function $k:\mathbb{N} \to \mathbb{N}$ such that $\lim_{n\to \infty}k(n)=\infty$, we get
	\begin{align*}
		\frac{R^{\alpha}_T}{n}  \le   \frac{1}{n}\Delta_{\alpha,\max} +     \frac{1}{n}\EEcc{ \sum_{i\in V_0}  \Delta_{\alpha,i}  \pa{\frac{4 \log T}{(\delta^{sup}_{\alpha,i}(k(n)))^2} + 8 }}{\mathcal{E}},
	\end{align*} where the expectation is taken over the random choice of $V_0$, and
	\begin{align*}
		\frac{R^{\alpha}_T}{n}
		\le 
			 9  \pa{\frac{ \EE{|C_1| }}{n} + 1 }  \left \lceil \frac{\log T}{\log (1/(1- \alpha))} \right\rceil \sqrt{T \log T}.    
	\end{align*}
\end{theorem}

For the proof of Theorem~\ref{supercrit_reg_unknown_l}, see Section~\ref{sec:irg_supercrit}.
Note that for a supercritical $\kappa$, $\EE{C_1} = \Theta_n(n)$. Therefore,  $R^{\alpha}_T$ scales linearly with $n$
and hence it is natural to normalize the regret by the number of nodes.
Both in the subcritical and supercritical regimes,
our bounds scale linearly with the maximal expected reward $c^*$, which is of $\Theta_n(1)$ in the 
subcritical case, but is $\Theta_n(n)$ in the supercritical case.
The dependence of the obtained bounds on the time horizon $T$ is similar in both regimes.
Note that unlike in the subcritical case, the censoring level $K$ is not a constant anymore as we
choose it to be $K=k(n)$ for some function $k$. Hence, strictly speaking, the feedback is not local as the number of vertices
that need to be explored is not independent of the number of nodes even if $k(n)$ can grow arbitrarily slowly.
Similarly to the subcritical case, a sufficiently large constant value of $K$ would suffice.
The value of the constant should be so large that for any vertex $i$, the conditional probability -- conditioned
on the event that $i$ is not in the giant component -- that the component of $i$ has size larger than $K$
is sufficiently small. Such a constant exists, see (\refeq{eq:dual}) below.
However, this
value depends on the unknown distribution of the underlying random graph. In the subcritical case
we solved this problem by applying a doubling trick. This is made possible by the fact that in the subcritical case
one observes the ``bad'' event that a component has size larger than $K$ and therefore censoring occurs.
By ``trying'' increasingly large values of $K$ one eventually finds a value such that the probability of censoring
is sufficiently small. However, in the supercritical case, the ``bad'' event is that even though the selected vertex
is not in the giant component, the size of the component is larger than $K$. Unfortunately, one cannot
decide whether the bad event occurs or simply the vertex lies in the giant component. For this reason, we have been
unable to apply an analogous doubling trick in the supercritical case. To circumvent this difficulty, 
we choose $K$ to be growing with $n$. This guarantees that the bad event occurs with small probability. The price to pay
is that the observation is not entirely local in the strict sense.

        \section{Degree observations}
\label{sec:degree}

The results of the previous section show that it is possible to learn
to maximize influence under very general conditions if the learner has
access to the censored size of the connected component, where
the size of censoring may be kept much smaller than the size of the
entire network. In this section we consider the case when the learner has
access to significantly less information. In particular, we study the
case when the learner only observes the degree of the selected vertex $A_t$
(i.e., the number of edges adjacent to $A_t$) in the graph $G_t$.
Under such a restricted feedback, one cannot hope to learn to maximize
influence in the full generality of sparse inhomogeneous random graphs
as in Section \ref{sec:censored}. However, we show that in several
well-known models of real networks, degree information suffices for
influence maximization. In particular, we study three random graph
models that have been introduced to replicate properties of large
(social) networks appearing in a variety of applications. These are
(1) stochastic block models; (2) the Chung--Lu model; and (3) Kronecker
random graphs.

\subsection{Three random graph models}

We start by introducing the three models we study.  All of them are special cases of inhomogeneous Erd\H{o}s--R\'enyi 
graphs.

\subsubsection*{Stochastic block model}  
In the stochastic block model, the probabilities $p_{i,j}$ are defined through the notion of 
\emph{communities}, defined as elements of a partition $H_1,\ldots,H_S$ of the set of vertices $V$. We refer to the index $m$ of 
community $H_m$ as the \emph{type} of a vertex belonging to $H_m$. Each community $H_m$ contains $\alpha_m n$ nodes (assuming 
without loss of generality that $\alpha_m n$ is an integer). With the help of the community structure, the probabilities $p_{i,j}$ 
are constructed as follows: if $i\in H_\ell$ and $j\in H_m$, the probability of $i$ and $j$ being connected is given by $p_{i,j} = 
\frac{K_{\ell,m}}{n}$, where $K$ is a symmetric matrix of size $S\times S$, with positive elements. The random graph from the above 
distribution is denoted as \sbm.

In the stochastic block model, identifying a node with maximal reward amounts to finding a node 
from the most influential community. Consequently, it is easy to see that choosing $\alpha$ such that $\alpha > 
\min_m \alpha_m$, the near-optimal set $V^*_{\alpha}$ exactly corresponds to the set of optimal nodes, and thus the quantile 
regret \eqref{eq:qregret} coincides with the regret \eqref{eq:regret}.

We consider the stochastic block models satisfying the following
simplifying assumptions:
\begin{assumption}\label{ass:1}
	$K_{l,m} = k > 0$ for all $l\neq m$.
\end{assumption}
This assumption requires that nodes $i,j$ belonging to different communities are connected with the same probability. Additionally, in our analysis in the supercritical case we make the following 
natural assumptions:
\begin{assumption}\label{ass:2}
	For all $l$, $ K_{l, l} > k$.
\end{assumption}
In plain words, this assumption requires that the density of edges within communities is larger than the density of 
edges between communities.

\subsubsection*{Chung--Lu model}
Another thoroughly studied special case of the inhomogeneous Erd\H{o}s--R\'enyi model is the 
so-called \emph{Chung--Lu model} (sometimes referred to as \emph{rank-1 model}) as first defined by \citet*{CL02} (see also 
\cite*{chung2006complex,Bollobas:2007:PTI:1276871.1276872}).
In this model the edge probabilities are defined by a 
vector $w\in\real^n$ with positive components, representing the ``weight'' of each vertex. Then the matrix defining the edge probabilities has entries $A_{ij} = w_i w_j$. We assume that the vector $w$ is such that $ w_i w_j/n < 1$ for all $i,j$. In other words, the Chung--Lu model 
considers rank-1 matrices of the form $A = ww\transpose$. The random graph from the Chung--Lu model is denoted by \chlu. 
Chung--Lu random graphs replicate some key properties of certain real networks. For instance, if $w$ is a sequence 
satisfying a power law, then \chlu is a power law model, which allows one to model social networks, see 
\cite*{chung2006complex}.

\subsubsection*{Kronecker graphs}

Kronecker random graphs were introduced by \cite*{10.1007/11564126_17, leskovec2008dynamics,leskovec2010kronecker} as models of large networks appearing in various applications, including social networks.
The matrix $P$ of the edge probabilities of a Kronecker random graph \krgraph is defined recursively.
The model is parametrized by the constants $\zeta, \beta, \gamma \in [0,1]$.
Here one assumes that the number of vertices $n$ is a power of $2$.
Starting from a $2\times 2$ seed matrix, 
$$
P^{[1]} = \begin{bmatrix}
	\zeta       & \beta \\
	\beta       & \gamma \\
      \end{bmatrix}~,
$$
we define the matrices $P^{[2]},\ldots,P^{[k]}$ such that for each $i=2,\ldots,k$,
$P^{[i]}$ is a $2^i \times 2^i$ matrix obtained from $P^{[i-1]} $ by
$$
P^{[i]} = \begin{bmatrix}
	\zeta P^{[i-1]}       & \beta P^{[i-1]} \\
	\beta  P^{[i-1]}      & \gamma P^{[i-1]} \\
      \end{bmatrix}~.
      $$
      Finally $P=P^{[k]}$. Hence, the Kronecker random graph \krgraph has $n=2^k$ vertices, where each vertex $i$ is characterised by a binary string $s_i\in \{0,1\}^k$, such that the probability of an edge between nodes $i$ and $j$  is equal to $p_{i,j}=\zeta^{\iprod{s_i}{s_j}}\gamma^{\iprod{\bar{1}-s_i}{\bar{1} - s_j}}\beta^{k - \iprod{s_i}{s_j} - \iprod{\bar{1}-s_i}{\bar{1} - s_j}}$,
  where $\bar{1}=(1,\ldots,1)\in \{0,1\}^k$ denotes the all-one vector and $\iprod{\cdot}{\cdot}$ is the usual inner product.
\cite*{10.1007/11564126_17} show that a Kronecker graph with properly tuned values of $\zeta, \beta, \gamma$  replicates properties of real world networks, such as  small
diameter, clustering, and heavy-tailed degree distribution.

\subsection{Learning with degree feedback in stochastic block models and
Chung--Lu graphs}

In this section we introduce an online influence maximization
algorithm that only uses the degree of the selected node as feedback information.
The algorithm is a variant
of the kl-UCB algorithm, that was proposed and analyzed by \cite*{GaCa11,maillard11dmed,cappe:hal-00738209,Lai87}. 
The main reason why learning is possible based on degree observations
only is that
nodes with the largest expected degrees $\mu^*$ are exactly 
the ones with the largest influence $c^*$. This (nontrivial) fact
holds in both the stochastic block model (under Assumptions
\ref{ass:1} and \ref{ass:2}) and the Chung--Lu model, across both the
subcritical and supercritical regimes. These facts are proven in
Sections \ref{sec:subcrit} and \ref{sec:supercrit}.
Further, we define $X_{t,i}$ as the degree of node $i$ in the realized graph $G_t$, and define $\mu_i = \EE{ X_{1,i}}$ as 
the expected degree of node $i$. We also define $c^* = \max_i c_i$ and $\mu^* = \max_i \mu_i$. 

\begin{algorithm}
	\caption{\ducbv}
	\label{alg:thealgexp}
	\textbf{Parameters:} A set of nodes $V_0 \subseteq V.$\\
	\textbf{Initialization:} Select each node in $V_0$ once. Observe the degree $X_{i,i}$
	of vertex $i$ in the graph $G_i$ for $i=1,\ldots,|V_0|$.
	For each $i\in V_0$, set $N_i(|V_0|)=1$ and ${\hmu}_i(|V_0|)=X_{i,i}$.
	\\
	\textbf{For} $t = |V_0|, \dots T$, \textbf{repeat}
	\begin{enumerate}
		\item For each node, compute 
		\[
		U_i(t) = \sup \left\{ \mu: \mu - {\hmu}_i(t) + {\hmu}_i(t) \log \left(\frac{{\hmu}_i(t)}{\mu} \right)\le 
		\frac{3\log(t)}{N_i(t)} 
		\right\}.
		\]
		\item Select any node $A_{t+1} \in \arg\max_{i} U_i(t)$. 
		\item Observe degree $X_{t+1,A_{t+1}}$ of node $A_{t+1}$ in $G_{t+1}$
		and update
		\[
		{\hmu}_{A_{t+1}}(t+1)=\frac{N_{A_{t+1}}(t){\hmu}_{A_{t+1}}(t+1)+X_{t+1,A_{t+1}}}{N_{A_{t+1}}(t)+1}~.
		\]
		Update $N_{A_{t+1}}(t+1)=N_{A_{t+1}}(t)+1$.
	\end{enumerate}
\end{algorithm}

The learner  uses the observed degrees as rewards, and feeds them to an instance of kl-UCB originally designed for 
Poisson-distributed rewards. A key technical challenge arising in the analysis is that the degree distributions do not actually belong to 
the Poisson family for finite $n$. We overcome this difficulty by showing that the degree distributions have a moment generating function 
bounded by those of Poisson distributions, and that this fact is sufficient for most of the kl-UCB analysis to carry through without 
changes.

As in the case of the inhomogeneous  Erd\H{o}s--R\'enyi model, we  subsample a set of size given in 
Equation~\eqref{eq:v0} of representative nodes 
for kl-UCB to play on.  For clarity of presentation, we first propose a simple algorithm that assumes prior 
knowledge of $T$, and then move on to construct a more involved variant that adds new actions on the fly.
We  present our \klucb variant for a fixed set of nodes $V_0$ as Algorithm~\ref{alg:thealgexp}. We refer to this algorithm as 
$\mbox{\ducb}(V_0)$ (short for ``degree-UCB on $V_0$''). Our two algorithms mentioned above use $\mbox{\ducb}(V_0)$ as a subroutine: they are both 
based on uniformly sampling a large enough set $V_0$ of nodes so that the subsample includes at least one node from the top 
$\alpha$-quantile, with high probability. 
We define the $\alpha$-optimal degree $\mu^*_\alpha = \min_{i\in V^*_{\alpha}} \mu_i$ and the  
gap parameter  $\delta_{\alpha,i} = \pa{\mu_i - \mu^*_\alpha}_+$. 
We first present a performance guarantee of our simpler algorithm that assumes knowledge of $T$, so the learner plays \ducbv on the uniformly sampled 
a subset of size (\ref{eq:v0}). 
This algorithm satisfies the following performance guarantee:
\begin{theorem}\label{thm:reg_knowT1}
  Assume that the underlying random graph is either (a) a subcritical
  stochastic block model satisfying Assumption \ref{ass:1}; (b) a supercritical
  stochastic block model satisfying Assumptions \ref{ass:1} and
  \ref{ass:2}; (c) a subcritical Chung--Lu random graph; or (d) a supercritical Chung--Lu random graph.
  
	Let $V_0$ be a uniform subsample of $V$ with size given in Equation~\eqref{eq:v0} and define the event $\mathcal{E} = \ev{V_0 \cap 
		V^*_\alpha \neq \emptyset}$.
	If the number of vertices $n$ is sufficiently large, then the expected $\alpha$-quantile regret of \ducbv simultaneously satisfies
	
	\begin{align*}
	R^{\alpha}_T \le  \EEcc{ \sum_{i \in V_0}  \Delta_{\alpha,i}  \pa{\frac{\mu_\alpha^*\pa{18 + 27\log T}}{\delta_{\alpha,i}^2} + 3} 
	}{\mathcal{E}}
	+ \Delta_{\alpha,\max},
	\end{align*}
	where the expectation is taken over the random choice of $V_0$, and
	\begin{align*}
	R^{\alpha}_T \le  
	18c^* \sqrt{\frac{T \mu^* \pa{2 + 3\log T}^2}{\log(1/(1- \alpha))}} + \pa{\frac{3 \log T}{\log(1/(1-\alpha))} +4} \Delta_{\alpha,\max}.
	\end{align*}
\end{theorem}
\begin{algorithm}[t]
	\caption{\ducb-\textsc{double}$(\beta)$}\label{alg:doubling}
	\textbf{Parameters:} 
	$\beta \ge 2$.\\
	\textbf{Initialization: $V_0 = \emptyset$.}
	\\
	\textbf{For} $k = 1, 2 \dots $, \textbf{repeat}
	\begin{enumerate}
		\item Sample subset of nodes $U_k$ uniformly such that $|U_k| = \left\lceil\frac{\log \beta}{\log(1/(1- \alpha))}\right\rceil$.
		\item Update action set $V_{k} = V_{k-1} \cup U_k$.
		\item For rounds $t = \beta^{k-1}, \beta^{k-1} + 1, \dots, \beta^{k} -1$, run a new instance of \ducb$(V_k)$.
	\end{enumerate}
\end{algorithm}

In contrast to the results obtained in the general setting of Section
\ref{sec:censored}, where we have to run different algorithms in the
subcritical and supercritical cases, for the models considered in this
section the learner can run the Algorithm~\ref{alg:doubling} without
prior knowledge of the regime.

For unknown values of $T$, we propose the \ducb-\textsc{double}$(\beta)$ algorithm (presented as Algorithm~\ref{alg:doubling}) that uses a 
doubling trick to estimate $T$. The following theorem gives a performance guarantee for this algorithm: 
\begin{theorem}\label{thm:doubling}
  Assume that the underlying random graph is either (a) a subcritical
  stochastic block model satisfying Assumption \ref{ass:1}; (b) a supercritical
  stochastic block model satisfying Assumptions \ref{ass:1} and
  \ref{ass:2}; (c) a subcritical Chung--Lu random graph; or (d) a supercritical Chung--Lu random graph.
  
	Fix $T$, let $k_{\max}$ be the value of $k$ on which \ducb-\textsc{double}$(\beta)$ terminates, and define the event $\mathcal{E} = 
	\ev{V_{k_{\max}} \cap V^*_\alpha = \emptyset}$.
	If the number of vertices $n$ is sufficiently large, then the $\alpha$-quantile regret of
	\ducb-\textsc{double}$(\beta)$ simultaneously satisfies
	
	\begin{align*}
	R^{\alpha}_T  \le  \EEcc{\sum_{i \in V_{k_{\max}}} \Delta_i \pa{ \pa{\frac{18\mu^*}{\delta_{\alpha,i}^2} + 3} (\log_\beta T + 1) 
			+  \frac{27 \log \beta (\log_\beta T+1)^2 }{ 2\delta_{\alpha,i}^2}}}{\mathcal{E}} + \Delta_{\alpha, \max} \log_\beta T ,
	\end{align*}
	where the expectation is taken over the random choice of the sets $V_1,V_2,\dots$, and
	\begin{align*}
	R^{\alpha}_T \le 36c^* \sqrt{\frac{T  \pa{\mu^* + \log \pa{\beta T}}\log^2T }{\log(1/(1- \alpha))}} 
	+ \pa{\frac{3 \log^2 T}{\log(1/(1-\alpha))} + 4} \Delta_{\alpha,\max}.
	\end{align*}
\end{theorem}

\subsection{Learning with degree feedback in Kronecker random graphs}

In this section we study influence maximization when the underlying
random network is a Kronecker random graph.
We set this model apart as the properties of Kronecker random graphs
differ significantly from those of the stochastic block model and the
Chung--Lu model.
At the same time, we show that observing the degree of the selected
nodes is enough to maximize the total influence in this graph model as
well. In particular, the same algorithm \ducbv introduced above achieves a small regret.

Since subcritical Kronecker random graphs contain only
$o(n)$ non-isolated vertices with high probability, we consider only supercritical regime with parameters are such that
$(\zeta+\beta)(\beta+\gamma)>1$.
Denote by $H$ the subgraph of \krgraph, induced by the vertices  of weight $l\ge k/2$. 
We exploit  the property that for the graph  \krgraph with parameters $(\zeta+\beta)(\beta+\gamma)>1$, there exists a constant $b(P)$ such that a
subgraph of \krgraph,  induced by the vertices of $H$, is connected
 with probability at least $1- n^{-b(P)}$, see \citet[Theorem
 9.10]{49013}.
 This means that on this event, the connected components $C_i$ are the
 same for all $i\in H$.  This allows us to prove the following:

\begin{theorem}\label{thm:kronecker}
	Let $V_0$ be a uniform subsample of $V$ of size  $\left\lceil\frac{\log (nT)}{\log(2)}\right\rceil$. Let \krgraph be such that $(\zeta+\beta)(\beta+\gamma)>1$ and $\zeta>\gamma> \beta$.
	Then there exists a constant $b(P)$ such that the quantile  regret of \ducbv  satisfies

	\begin{align*}
		\frac{R_{T}^{ \alpha}}{n} &\le      \left\lceil\frac{\log (nT)}{\log(2)}\right\rceil   \pa{
			\frac{\mu^*(2 + 6\log T)}{\pa{1-\frac{\beta+\gamma}{\zeta+\beta}}^2} + 3    +n^{-b(P)} 
		\frac{\mu^*(2 + 6\log T)(\zeta+\beta)^2}{(\zeta - \gamma)^2} + 3n^{-b(P)} } + 1~.
	\end{align*}
\end{theorem}

	\section{Discussion}\label{sec:conc}

In this section we highlight some features of our results and discuss directions for future work. Our main results show
that online influence maximization is possible with only local feedback information. We establish 
bounds for the quantile regret that are polylogarithmic in $T$ for all considered random graph models.
Notably, our bounds hold for both the subcritical and supercritical regimes of the random-graph models considered, and show no explicit dependence on the number of nodes $n$.

\paragraph{Previous work.} Related online influence maximization algorithms consider more general classes of networks, but make more restrictive assumptions about the 
interplay between rewards and feedback. The line of work explored by \citet*{WeKvVaVa17,WaCh17} assumes that the algorithm receives 
\emph{full feedback} on where the information reached in the previous trials (i.e., not only the number of influenced nodes, but their 
exact identities and influence paths, too). Clearly, such detailed measurements are nearly 
impossible to obtain in practice, as opposed to the local observations considered in this paper.

Another related setup was considered by \citet*{CaVa16}, whose algorithm only receives feedback about the nodes that were 
directly influenced by the chosen node, but the model does not assume that neighbors in the graph share the information to further 
neighbors and counts the reward only by the nodes directly connected to the selected one. That is, in contrast to our work, this work does 
not attempt to show any relation between local and global influence maximization.
One downside to all the above works is that they all provide rather conservative performance guarantees:
On the one hand, \citet*{WeKvVaVa17} and~\citet*{CaVa16}  are concerned with worst-case regret bounds that 
uniformly hold for all problem instances for a fixed time horizon $T$. On the other hand, the bounds of \citet*{WaCh17} depend on 
topological (rather than probabilistic) characteristics of the underlying graph structure, which inevitably leads to conservative 
results. For example, their bounds instantiated in our graph model lead to a regret bound of order $n^3 \log T$, which is virtually void of 
meaning in our regime of interest where $n$ is very large (e.g, much larger than $T$). In contrast, our bounds do not show any explicit 
dependence on $n$. In this light, our work can be seen as the first attempt that takes advantage of 
specific probabilistic characteristics of the mechanism of information spreading to 
obtain strong instance-dependent global performance guarantees, all while having access to only local observations.

Other related framework is stochastic online learning under partial monitoring \citet*{ pmlr-v9-agarwal10a, bartok2012adaptive, 2969239.2969439}. In this setting the loss is not directly observed by the learner, which makes this setting applicable to a wider range of problems. However, the partial monitoring setup is too general to capture the specific relationship of feedback and influenced component size, resulting to regret bounds that scale with $n^2$.

\paragraph{Tightness of the regret bounds.}
In terms of dependence on $T$, both our instance-dependent and worst-case bounds are near-optimal in their respective 
settings: even in the simpler stochastic multi-armed bandit problem, the best possible regret bounds are $\Omega_T(\log T)$ and 
$\Omega_T(\sqrt{T})$ in the respective settings \citet*{auer2002finite,auer2002bandit,bubeck12survey}. The optimality of our bounds with 
respect to other parameters such as $c^*$, $\mu^*$ and $n$ is less clear, but we believe that these factors cannot be improved 
substantially for the models that we studied in this paper. As for the subproblem of identifying nodes with the highest degrees, we believe 
that our bounds on the number of suboptimal draws is essentially tight, closely matching the classical lower bounds by \citet{LR85}.

	\section{Multi-type branching processes}
\label{sec:branching}

One of the most important technical tools for analyzing the component structure of
random graphs is the theory of \emph{branching processes}, see \citet*{Bollobas:2007:PTI:1276871.1276872, hofstad_2016}.
Indeed, while the connected components of an inhomogenous random graph \irg have a complicated 
structure, many of their key properties may be analyzed through the concept of
multi-type Galton--Watson processes. Recall the notation introduced in Section \ref{sec:randomgraphmodel}.
Consider a Galton--Watson process, where an individual $x\in (0,1]$ is replaced in the next generation by a set of particles
distributed as a Poisson process on $(0,1]$ with intensity $\kappa(x,y)d\mu(y)$
and the number of
children has a Poisson distribution with mean $\int_{(0,1]}\kappa(x,y)d\mu(y)$. We denote this branching process, started with a single particle $x$ by $W_{\kappa}(x)$. \cite*{Bollobas:2007:PTI:1276871.1276872} establishes a connection between the sizes of
connected components of \irg, the 
survival probability of a  branching process $\GW_{\kappa}(x)$, and the function $\kappa$. 
As shown in \cite*{Bollobas:2007:PTI:1276871.1276872}, the operator $\Phi_{\kappa}$
can be directly used for characterizing the probability 
$\rho(x)$ of survival of the process $\GW_{\kappa}(x)$ for all $x \in (0,1]$.
By their Theorem 6.2, the function $\rho$ is the maximum fixed point of the non-linear  equation $\Phi_{\kappa}(f) = f$. Furthermore, as was shown in \citet*[Lemma 5.8.]{Bollobas:2007:PTI:1276871.1276872}, if $\|T_{\kappa}\|_2 < 1$, then $\rho(x) = 0$ for all $x$ and when $\|T_{\kappa}\|_2 > 1$,  $\rho(x)>0$ for all $x$.

To analyze the random graph  \irg, we use Poisson multi-type Galton--Watson branching processes with $n$ types, parametrized by an $n \times n$ matrix $A$ with positive elements. Therefore, each node corresponds to its own type. The branching process tracks the evolution of a set of \emph{individuals} of various types. Starting in round 
$n=0$ from a single individual of type $i$, each further generation in the Galton--Watson process $\GW_{\kappa}(i)$ is generated by each 
individual of each type $i$  producing $X_{i,j}\sim Poisson(A_{i,j}/n)$ new individuals of each type $j$. Therefore, the number of  offsprings of 
the individual of type $i$ is $\sum_{j=1}^n X_{i,j} \sim Poisson( \sum_{j=1}^n A_{i,j}/n)$. 

Our analysis below makes use of the following quantities associated with the multi-type
branching process:
\begin{enumerate}
	\item $Z_n(i)$ is the number of individuals in generation $n$ of $\GW_{\kappa}(i)$ (where $ Z_0(i) = 1$);
	\item $B(i)$ is the \emph{total progeny}, that is, the total number of individuals generated by $\GW_{\kappa}(i)$ and its expectation is denoted by $x_i = \EE{B(i)}$;
	\item $\rho(i)$ is the \emph{probability of survival}, that is, the probability that $B(i)$ is infinite.
\end{enumerate}

	\section{Proofs of Theorem~\ref{subcrit_reg} and \ref{subcrit_reg_doubling}. }
\label{component_concentration}
The connected components $C_i$ of an individual  $i$ have a complicated structure, but many key properties can be analyzed through the concept of multi-type Galton-Watson branching processes with $n$ types. Fix an arbitrary node $i$ and let $Y_{i,1}, Y_{i,2}, \dots, Y_{i,n}$ be independent Bernoulli random 
variables with respective parameters $A_{i,j}/n$ for $i,j\in [n]$. 
Consider a multitype binomial branching process where an individual of type $i$ produces an individual $j$ with probability $A_{i,j}/n$, and 
let $B_{Ber}(i)$ denote its total progeny when started from an individual $i$.  In the same way, consider a multitype Poisson branching process where an individual of type $i$ produces  $X_{i,j} \sim Poisson(A_{i,j}/n)$ individuals, and 
let $B(i)$ denote its total progeny when started from an individual $i$.  
We use the concept of \emph{stochastic dominance} between random variables.
The random variable $X$ is \emph{stochastically dominated} by the random variable $Y$ when, for every $x \in\mathbb{R}$, 
$\PP{X \le x} \ge \PP{Y \le x}$. We denote this by $X  \preceq Y$. 

\paragraph{Proof of Lemma~\ref{subcritical}.}
	First, we define an upper approximation to $\kappa$. We choose
        an integer $m$ and we partition the interval $(0,1]$ into $m$
        sets $\mathcal{A}_1, \dots, \mathcal{A}_{m}$, where
        $\mathcal{A}_k = ((k-1)/m, k/m] $, $k \in [1, m]$. Also we
        denote by
        $\mathcal{A}_m(x)$ the set $\mathcal{A}_k$ for which $x\in  \mathcal{A}_k$. Then we bound $\kappa$ from above by
	\[ \kappa^+_m(x,y) =  \sup\{ \kappa(x', y'): x'\in  \mathcal{A}_m(x), y' \in \mathcal{A}_m(y)  \}~.   \]
	As $\kappa$ is bounded, there exists a sufficiently large $m$
        such that $\|T_{\kappa^+_m}\|<1$:
        \[
          \|T_{\kappa^+_m}\| \le \|T_{\kappa}\| + \|T_{\kappa^+_m} -
          T_{\kappa}\| \le  \|T_{\kappa}\|  +
          \pa{\int_{(0,1]\times(0,1]}(\kappa^+_m(x,y) - \kappa(x,y))^2
            dx dy}^{1/2}~.
        \]

    Then for any node $i$ in \irg, we define a type $k_i = k$ if $ (k-1)/m < i/n \le k/m $ holds. By our definition of $\kappa^+_m$, we have  $$ \PP{|C_i| > u} \le \PP{B_{Ber}(k_i) > u}.$$ 
	For $k, \ell \in [m]$ we define $p_{k,\ell} = \frac{1}{m} \kappa_m^+(k/m, \ell/m)$.
	Notice, that for random variables $Y\sim Ber(p)$ and $X\sim Poisson(p^{\prime})$ with $p^{\prime} = -\log(1-p)>p$, $Y \preceq X$ holds. This follows from the observation that $\PP{Y > 0} = p$ and $\PP{X>0} = p$. 
	It follows that $Ber(p_{k,\ell})  \preceq Poisson((1+\varepsilon)p_{k,\ell})$. 

	Then there exists $\varepsilon>0$ such that the multitype
        Poisson  branching process $\widetilde{B}(k)$ with parameters
        $(1+\varepsilon)p_{k,\ell}$ is such that $\PP{B_{Ber}(k) > u}
        < \PP{\widetilde{B}(k) > u}$ and it is subcritical.
        We also define a random variable $\widetilde{X}_{k,\ell} \sim Poisson\pa{(1+\varepsilon)p_{k,\ell}}$.
	Since the total number of descendants of individuals in the first generation are independent, we can write the following recursive equation on the number of descendants of type $k$: 
	$$|\widetilde{B}(k)| = 1 + \sum_{\ell =1}^m \widetilde{X}_{k,\ell} |\widetilde{B}(\ell)|.$$
	For any type $k$, for  $z_k > 1$,  the probability generating function of $|\widetilde{B}(k)|$ is $g(k) = \EE{z_k^{|\widetilde{B}(k)|}}$ and we denote $g = (g(1), \dots, g(m))^T$.
	Using that for $X\sim Poisson(\gamma)$ for some $\gamma>0$,
        $y>1$ the probability generating function  is $\EE{y^{X}} = e^{\gamma(y-1)}$, we have
	\begin{align*}
	g(k) =  \EE{z_k^{|\widetilde{B}(k)|}} = z_k \EE{z_k^{\widetilde{X}_{k,1} |\widetilde{B}(1)|} \dots z_k^{\widetilde{X}_{k,\mathcal{M}} |\widetilde{B}(m)|}} = z_k \prod_\ell \EE{z_k^{\widetilde{X}_{k,\ell} |\widetilde{B}(\ell)|}} \\
	= z_k \prod_\ell \EE{ \big( \EE{z_k^{|\widetilde{B}(\ell)|} }  \big)^{\widetilde{X}_{k,\ell}} } = z_k \exp\bigg((1+\varepsilon)\sum_\ell p_{k,\ell}(g(\ell) - 1)\bigg).
	\end{align*}
	Recall that $P$ denotes the $m\times m$ matrix with entries $p_{k,\ell}$. Our next aim is to study the fixed point of the operator $G_P$, defined as
	\begin{equation}\label{fixedpoint}
	g = G_P g := z\exp\bigg((1+\varepsilon)P(g-\bar{1})\bigg)~.
	\end{equation}
	Define the function $F(z,g) =
        z\exp\bigg((1+\varepsilon)P(g-\bar{1})\bigg) - g$. This
        function is  smooth and the entries of the  Jacobian matrix are $$J_{k,\ell}(z,g) := \frac{\partial F_k}{\partial g_\ell} = z_k(1+\varepsilon)p_{k,\ell}\exp\pa{ (1+\varepsilon)\sum_\ell p_{k,\ell}(g_\ell-1)}-\II{k=\ell}.$$ 
	Let $P'(g,z)$ be the matrix with elements
        $z_k(1+\varepsilon)p_{k,\ell}\exp\pa{ (1+\varepsilon)\sum_\ell
          p_{k,\ell}(g_\ell-1)}$. Then, at point $(\bar{1}, \bar{1})$,
        $P'_{k,\ell}(\bar{1}, \bar{1}) = (1+\varepsilon)p_{k,\ell}
        $. Since  $\varepsilon$ is chosen such that the branching
        process $\widetilde{B}(k)$ is subcritical, $P'(\bar{1},
        \bar{1})$ is smaller than one. This means, that we can find
        $z' = 1+ \delta, g' > 0$, such that the largest eigenvalue of
        $P'(g', z')$ is smaller than one as well, and therefore $J(z',
        g')$ is invertible. Then, by the implicit function theorem there exists an open set $U_z \subset (1,+\infty)^m$ and a function $q: U_z \to (0, +\infty)^m$ such that $F(z,q(z)) =\bar 0$. 

        Finally, the statement of the lemma is obtained by applying the Chernoff bound:
	$$ \PP{\widetilde{B}(k) > u} =  \PP{z_k^{\widetilde{B}(k)} > z_{k}^u}\le \frac{\EE{z_{k}^{\widetilde{B}(k)}}}{z_k^u}.  $$
	Denote $\lambda_k =  \ln(z_{k})>0$. Then, 
	$$\frac{\EE{z_{k}^{\widetilde{B}(k)}}}{z_{k}^u} = \frac{g_{ k}}{z_{k}^u} =   \exp(-\lambda_k u)g_{k} .$$
 Then taking any $\lambda(\kappa) = \min_k \lambda_k$, $g(\kappa)= \max_k g_k$, we get the statement of the lemma.
\qed

Armed with this concentration result, we can see that the typical size $|C_i|$
of the connected component of any vertex $i$ is $O(1)$.
Recall that the learning algorithm has only access to a censored value
of $|C_i|$, truncated by a constant $K$.
 Our main technical result shows that 
nodes with the largest expected censored observations $ \musub_*(K)$ are exactly 
the ones with the largest influence $c_*$. 
We formally state this result next:

\begin{lemma}\label{prop:irg_subcr}
	 For  \irg with subcritical $\kappa$,  and $n>n_0(\kappa)$,   for any  node $i$
	 we have $c_* - c_i \le	\musub_*(K) - \musub_i(K) +   e^{-\lambda(\kappa) K}g$. 
	 Then, for $K = \frac{\log T}{\lambda}$, with $\lambda < \lambda(\kappa)$ we have $c_* - c_i \le	\musub_*(K) - \musub_i(K) +  \frac{g(\kappa)}{T}$.
\end{lemma}

\begin{proof}
	
	The expected bias of $ \musub_i(K) $ is, using the result of Lemma~\ref{subcritical}: 
	\begin{align*}
	c_i -  \musub_i(K)  =\EE{|C_i(G_t)| -  \musub_i(K)  }=\EE{(|C_i| - K)_+} \\ \le \int_0^{\infty} \PP{|C_i| - K > u}du \le \int_0^{n} e^{-\lambda(u + K)} du \le e^{-\lambda K}g(\kappa).
	\end{align*}
	
	Set $K = \frac{\log T}{\lambda }$.  Then,
	\begin{align*}
	c_* - c_i \le 	u_*(\log T/\lambda) -  \musub_i(\log T/\lambda )  +  \frac{g(\kappa)}{ T}.
	\end{align*}
	
\end{proof}

\paragraph{Proof of Theorem~\ref{subcrit_reg}.}
In order not to overload notation we write $\delta^{sub}_{\alpha,i}$ for $\delta^{sub}_{\alpha,i}(K)$.
We first note that, with high probability, the size of $V_0$ guarantees that the subset contains at least one node from the set 
$V^*_{\alpha}$: $\PP{\mathcal{E}} \ge 1- 1/T$.
Then, the regret can be bounded as
\begin{align}\label{conditional_regret}
R^{\alpha}_T
&\le \PP{\mathcal{E}^c} T \Delta_{\alpha,\max} + 
\EEcc{\sum_{t=1}^{T} \sum_{i \in V_0} \mathbb{I}[A_t = i]  \Delta_{\alpha,i}}{\mathcal{E}} \PP{\mathcal{E}}
\\
&\le 
\Delta_{\alpha,\max} + \EEcc{ \sum_{i\in V_0}  \Delta_{\alpha,i}   \EE{N_{i, T}} }{\mathcal{E}}~.
\end{align}

By Hoeffding’s inequality,
$$\PPcc{A_{t+1} = i}{N_{i,t} \ge \frac{4K^2 \log t}{(\delta^{sub}_{\alpha,i})^2}} \le \frac{4}{t^2}.$$
Then, 
\begin{align*}
\EE{N_{i,T}} \le  \frac{4K^2 \log T}{(\delta^{sub}_{\alpha,i})^2} + \sum_{t=|V_0|}^{T} \PPcc{A_{t+1} = i}{N_{i,t} \ge \frac{4K^2 \log t}{(\delta^{sub}_{\alpha,i})^2}} \\ \le  \frac{4K^2 \log T}{(\delta^{sub}_{\alpha,i})^2} + \sum_{t=|V_0|}^{T} \frac{4}{t^2} \le \frac{4K^2 \log T}{(\delta^{sub}_{\alpha,i})^2} + 8.
\end{align*}

Now, observing that $\delta^{sub}_{\alpha,i} \le \max_{j \in V_0} \musub_j(K) - \musub_i(K)$ holds under event $\mathcal{E}$, we obtain
\begin{equation}\label{regret_goodevent}
R^{\alpha}_T \le   \Delta_{\alpha,\max} +  \EEcc{ \sum_{i\in V_0}  \Delta_{\alpha,i}  \pa{
		\frac{4K^2 \log T}{(\delta^{sub}_{\alpha,i})^2} + 8 }}{\mathcal{E}},
\end{equation}
thus proving the first statement.

Next, we turn to proving the second statement regarding worst-case guarantees. To do this, we appeal to 
Proposition~\ref{prop:irg_subcr} and take $K= \frac{\log T}{\lambda}$, where $\lambda$ is any number, satisfying conditions of Lemma~\ref{subcritical}.
To proceed, let us fix an arbitrary $\varepsilon > 0$ and
split the set $V_0$ into two subsets: $U(\varepsilon) = \ev{a\in V_0: \delta^{sub}_{\alpha,i} \le \varepsilon}$ and $W(\varepsilon) = V_0 
\setminus U(\varepsilon)$.  Then, under event $\mathcal{E}$, we have
\begin{align*}
\sum_{i\in V_0}  \Delta_{\alpha,i}   \EE{N_{i,T}} &= \sum_{i\in U(\varepsilon)}  \Delta_{\alpha,i}   \EE{N_{i,T}} +
\sum_{i\in W(\varepsilon)}  \Delta_{\alpha,i}   \EE{N_{i,T}}
\\
&\le \varepsilon \sum_{i\in U(\varepsilon)} \EE{N_{i, T}} +  \frac{g}{ T} \sum_{i\in U(\varepsilon)} \EE{N_{i,T}} 
\\
&+  \sum_{i\in W(\varepsilon)}  \delta^{sub}_{\alpha,i} \pa{
	\frac{4\pa{\frac{\log T}{\lambda}}^2 \log T}{(\delta^{sub}_{\alpha,i})^2} } +  \frac{g}{ T} \sum_{i\in W(\varepsilon)}  
\frac{4\pa{\frac{\log T}{\lambda}}^2 \log T}{(\delta^{sub}_{\alpha,i})^2}\\
& +8 |W(\varepsilon)|\Delta_{\alpha,\max} 
\\
&\le \varepsilon T +  g +  \sum_{i\in W(\varepsilon)}   \pa{\frac{4\pa{\frac{\log T}{\lambda}}^2 \log T}{\delta^{sub}_{\alpha,i}} } +  \frac{g}{T} \sum_{i\in W(\varepsilon)}  
\frac{4\pa{\frac{\log T}{\lambda}}^2 \log T}{(\delta^{sub}_{\alpha,i})^2}\\
&  +8 |W(\varepsilon)|\Delta_{\alpha,\max} 
\\
&\le  \varepsilon T +g+ |V_0|   \frac{4\pa{\frac{\log T}{\lambda}}^2 \log T}{\varepsilon}  + \frac{g}{T}|V_0|  
\frac{4\pa{\frac{\log T}{\lambda}}^2 \log T}{\varepsilon^2} + 8 |V_0| \Delta_{\alpha,\max}
\\
&\le 4\pa{\frac{\log T}{\lambda}}\sqrt{|V_0|T \log T  } + 2g+ 8 |V_0| \Delta_{\alpha,\max}.
\end{align*}
where the last step uses the choice $\varepsilon = 2\pa{\frac{\log T}{\lambda}}\sqrt{|V_0| \log T / T}$. Plugging in the choice of $|V_0|$ concludes 
the proof.
\qed

\paragraph{Proof of Theorem~\ref{subcrit_reg_doubling}.}
To simplify the notation, we use $\lambda $ instead of $\lambda(\kappa)$. Let $T_q$ be the length of the $q$-th iterate. 
The expected regret over each period $q$ can be bounded as an expected regret of Local UCB($V_0$) with parameters $\lambda_q = 2^{-q}$ and $T_q$ time steps. Appealing to Theorem~\ref{subcrit_reg}, we can bound the expected regret as
\begin{align*}
	R^{\alpha}_T & \le \PP{\mathcal{E}^c} T \Delta_{\alpha, \max} + \EEcc{ \sum_{t=1}^T \sum_{i \in V_0}  \II{A_t = i}\Delta_{\alpha,i} }{\mathcal{E}}\\
	&\le \Delta_{\alpha, \max} + \EEcc{ \sum_{q= 1}^{Q_{\max}}\sum_{i \in V_0} \Delta_{\alpha,i} \EE{N_{i, T_q}}  }{\mathcal{E}}
\end{align*}

Following the analysis of Theorem~\ref{subcrit_reg}, by (\ref{regret_goodevent}), we get
\begin{align*}
	R^{\alpha}_T & 
	\le \Delta_{\alpha, \max}  +  \EEcc{ \sum_{q= 1}^{Q_{\max}} \sum_{i\in V_0}  \Delta_{\alpha,i}  \pa{
			\frac{4K_q^2 \log T}{(\delta^{sub}_{\alpha,i}(K_q))^2} + 8 }  }{\mathcal{E}}.
\end{align*}

We have $Q_{\max} = \lceil \log_2(1/\lambda) \rceil \le  \log_2(1/\lambda) + 1$, and 
 $$\sum_{q= 0}^{Q_{\max}} K_q^2 = \log^2 T  \sum_{q= 0}^{Q_{\max}} 4^q = \log^2 T \frac{ 4^{Q_{\max} + 1} - 1  }{3} \le \frac{16}{3} \frac{1}{\lambda^2}\log^2 T.$$
This gives us
\begin{align*}
	R^{\alpha}_T & 
	\le \Delta_{\alpha, \max}  +  \frac{64}{3} \EEcc{ \sum_{i\in V_0}  \Delta_{\alpha,i}  \pa{
			\frac{ \log^3 T}{(\lambda \cdot \min_{q\in[Q_{\max}]}\{\delta^{sub}_{\alpha,i}(K_q)\})^2} + 8 }  }{\mathcal{E}}.
\end{align*}

Next, we prove the second statement regarding worst-case guarantees. 
To proceed, let us take  $\varepsilon_q =\frac{\log T}{\lambda}\sqrt{|V_0| \log T / T_q}$ and
split the set $V_0$ into two subsets: $U(\varepsilon_q) = \ev{a\in V_0: \delta^{sub}_{\alpha,i}(K_q) \le \varepsilon_q}$ and $W(\varepsilon_q) = V_0 
\setminus U(\varepsilon_q)$.   Then, under event $\mathcal{E}$, we have
\begin{align*}
	 &\EE{ \sum_{t=1}^T \sum_{i \in V_0}  \II{A_t = i}\Delta_{\alpha,i} }\\
	 &\qquad \le  \EE{ \sum_{q= 1}^{Q_{\max}} \sum_{t_q = 1}^{T_q} \sum_{i \in V_0} \Delta_{\alpha,i} \II{|C_{A_{t_q}}| \le K_q} } + \EE{ \sum_{q= 1}^{Q_{\max}} \sum_{t_q = 1}^{T_q} \sum_{i \in V_0} \Delta_{\alpha,i} \II{|C_{A_{t_q}}| > K_q} } \\&
	 \qquad= \underbrace{\EE{ \sum_{q= 1}^{Q_{\max}} \sum_{t_q = 1}^{T_q} \sum_{ i\in U(\varepsilon_q)} \Delta_{\alpha,i} \II{|C_{A_{t_q}}| \le K_q} }}_{\text{Term 1}} + \underbrace{\EE{ \sum_{q= 1}^{Q_{\max}} \sum_{t_q = 1}^{T_q} \sum_{ i\in U(\varepsilon_q)} \Delta_{\alpha,i} \II{|C_{A_{t_q}}| > K_q} }}_{\text{Term 2}}
	 \\&
	 \qquad+ \underbrace{\EE{ \sum_{q= 1}^{Q_{\max}} \sum_{t_q = 1}^{T_q} \sum_{i\in W(\varepsilon_q)} \Delta_{\alpha,i} \II{|C_{A_{t_q}}| \le K_q} }}_{\text{Term 3}} + \underbrace{\EE{ \sum_{q= 1}^{Q_{\max}} \sum_{t_q = 1}^{T_q} \sum_{ i\in W(\varepsilon_q)} \Delta_{\alpha,i} \II{|C_{A_{t_q}}| > K_q} }}_{\text{Term 4}}.
\end{align*}
Term 1: 
\begin{align*}
	 \EE{ \sum_{q= 1}^{Q_{\max}} \sum_{t_q = 1}^{T_q} \sum_{ i\in U(\varepsilon_q)} \Delta_{\alpha,i} \II{|C_{A_{t_q}}| \le K_q} } \le |V_0|\frac{\log T}{\lambda}\EE{ \sum_{q= 1}^{Q_{\max}} \sqrt{|V_0| T_q \log T}  } 
	 \\ \le  |V_0|^{3/2}\frac{\log T}{\lambda} \sqrt{(\log_2(1/\lambda)+1) T}.
\end{align*}
Term 2:
The expected bias of $\mu_{i,t}^{sub}(K_q)$ is, using the result of Lemma~\ref{subcritical}: 
\begin{align*}
	\EE{(|C_i| - K_q)_+} & \le \int_0^{\infty} \PP{|C_i| - K_q > u}du \le \int_0^{n} e^{-\lambda(u + K_q)} du \le e^{-\lambda K_q}g 
	\\&= \pa{\frac{1}{T}}^{\lambda 2^q} g \le \pa{\frac{1}{T}}^{2^{q - Q_{max}}}g.
\end{align*}

Then,
\begin{equation}\label{bias_q}
	c_* - c_i \le 	\mu^{sub}_*(K_q) - \mu^{sub}_i(K_q) + \pa{\frac{1}{T}}^{2^{q - Q_{max}}}g.
\end{equation}
According to the stopping rule, we get 
\begin{align*}
	& \EE{ \sum_{q= 1}^{Q_{\max}} \sum_{t_q = 1}^{T_q} \sum_{ i\in U(\varepsilon_q)} \Delta_{\alpha,i} \II{|C_{A_{t_q}}| > K_q} } \le|V_0| \EE{ \sum_{q= 1}^{Q_{\max}}  \pa{ \varepsilon_q + g\pa{\frac{1}{T}}^{2^{q - Q_{max}}} } \pa{\frac{1}{T} + \sqrt{\frac{\log T}{2 T_q}} }T_q }
	\\
	&\qquad \le |V_0| \EE{ \sum_{q= 1}^{Q_{\max}}  \pa{ \frac{\log T}{\lambda}\sqrt{|V_0| \log T / T_q} + g} \pa{\frac{1}{T} + \sqrt{\frac{\log T}{2 T_q}} }T_q }
	\\
	&\qquad \le |V_0|  \pa{ \frac{\log T}{\lambda} \sqrt{\frac{(\log_2(1/\lambda + 1))|V_0| \log T }{T}} + \frac{g}{T}}     +  |V_0|^{3/2} (\log_2(1/\lambda)+1) \frac{\log^2 T}{\lambda} 
	\\
	&\qquad + |V_0|  g \pa{ \sqrt{(\log_2(1/\lambda)+1)T\log T}\frac{\log T}{\lambda}  } . 
\end{align*}
Term 3: Following the analysis of Theorem~\ref{subcrit_reg} and by (\ref{bias_q}), we get
\begin{align*}
 &\EE{ \sum_{q= 1}^{Q_{\max}} \sum_{t_q = 1}^{T_q} \sum_{i \in W(\varepsilon_q)} \Delta_{\alpha,i} \II{A_t = i, |C_{A_t} |\le K_q|} }  \\
&\qquad \le  \EE{\sum_{i\in W(\varepsilon_q)} \sum_{q= 0}^{Q_{\max}}  \delta^{sub}_{\alpha,i}(K_q) \pa{ \frac{4\pa{\frac{\log T_q}{\lambda}}^2 \log T}{(\delta^{sub}_{\alpha,i}(K_q))^2} } }  +8 |V_0|\Delta_{\alpha,\max} 
\\
&\qquad \le 4\sqrt{|V_0|} \EE{ \sum_{q= 0}^{Q_{\max}}   \pa{ \sqrt{T_q}\frac{\log^{3/2} T}{\lambda}  } }  +8 |V_0|\Delta_{\alpha,\max} 
\\
&\qquad \le 4\sqrt{|V_0|}   \pa{ \sqrt{(\log(1/\lambda) +1 )T}\frac{\log^{3/2} T}{\lambda}  }   +8 |V_0|\Delta_{\alpha,\max}. 
\end{align*}
Term 4: 
\begin{align*}
	&\EE{ \sum_{q= 1}^{Q_{\max}} \sum_{t_q = 1}^{T_q} \sum_{i \in W(\varepsilon_q)} \Delta_{\alpha,i} \II{A_t = i, |C_{A_t} |> K_q|} }  
	\\
	&\qquad \le  \EE{\sum_{i\in W(\varepsilon_q)} \sum_{q= 0}^{Q_{\max}}  \pa{\delta^{sub}_{\alpha,i}(K_q) + g\pa{\frac{1}{T}}^{2^{q-Q_{\max}}} } \pa{
			\frac{4\pa{\frac{\log T_q}{\lambda}}^2 \log T}{(\delta^{sub}_{\alpha,i}(K_q))^2} } \pa{ \frac{1}{T} + \sqrt{\frac{\log T}{2 T_q}}   }}  
		\\ 
		&\qquad   +8 |W(\varepsilon_q)|\Delta_{\alpha,\max} 
		\\&\qquad \le 4 \sqrt{|V_0|} (\log_2(1/\lambda)+1)   
				\frac{\log ^{3/2}T}{\sqrt{T} \lambda}     +  4\sqrt{|V_0|}   (\log_2(1/\lambda)+1)
				\frac{\log^2 T}{\lambda}  	
		\\ 
		&\qquad  + 4 g(\log_2(1/\lambda)+1)\frac{\log T}{\lambda}  +  4 g  	\frac{\log^{3/2} T}{\lambda} \sqrt{(\log_2(1/\lambda)+1) T }   +8 |V_0|\Delta_{\alpha,\max}.
\end{align*}

Putting everything together, we conclude that
\begin{align*}
	&R_t^{\alpha} \le 4\frac{1}{\sqrt{\ln(1/(1-\alpha))}}   \pa{ \sqrt{(\log(1/\lambda) +1 )T}\frac{\log^{2} T}{\lambda}  }   +16 \Delta_{\alpha,\max} \frac{\log T}{\sqrt{\ln(1/(1-\alpha))}}   \\
	&\qquad 
	+  4 g(\log_2(1/\lambda)+1)\frac{\log T}{\lambda}  +  4 g  	\frac{\log^{3/2} T}{\lambda} \sqrt{(\log_2(1/\lambda)+1) T }    
	\\
	&\qquad + \log T   \cdot \sqrt{\frac{\log_2(1/\lambda + 1) }{T \ln(1/(1-\alpha))  }} + \frac{g \sqrt{\log T}}{T\sqrt{ \ln(1/(1-\alpha))} }    +  \frac{2 (\log_2(1/\lambda)+1) \log^{5/2} T}{\pa{\ln(1/(1-\alpha))}^{3/2}} \\
	&\qquad 
	 +   g  \sqrt{\frac{T (\log_2(1/\lambda)+1)}{\ln(1/(1-\alpha))}} \log T.
\end{align*}

	\section{Proof of Theorem~\ref{supercrit_reg_unknown_l}. }
\label{sec:irg_supercrit}

The proof relies on some known properties of the largest connected
component in \irg for supercritical $\kappa$. 
We denote the largest and second-largest connected components of $G_t$ by
$C_1(G_t)$ and $C_2(G_t)$, respectively. The survival probability of
the branching process $W_{\kappa}(x)$ is denoted as $\rho(x)$. The
expected size of the connected component containing vertex $i$ can be
estimated in terms of  $\rho(i/n)$ and $\EE{|C_1|}$ as
$$c_i = \rho(i/n)\EE{|C_1|}+o_n(n)~,
$$
see \citet*[Chapter 9]{Bollobas:2007:PTI:1276871.1276872}.
The following properties are proved by \citet*{Bollobas:2007:PTI:1276871.1276872}:
\begin{itemize}
	\item If \irg is supercritical, then, with high probability, $C_1 = \Theta_n(n)$;
	\item $C_1(G_n) \to \sum_{i \in V} \rho(i/n)$ in probability;
	\item $C_2(G_n) = o_n(n)$ with high probability.
\end{itemize}

Recall from Section~\ref{IRG} that in the supercritical case the
feedback $\musup_{i, t}(K)$ is the indicator whether   $|C_i|$ is
larger than $K$. In the following lemma we show that taking $K = k(n)$
for an arbitrary function of $n$ that diverges to infinity, it is enough to control the bias of the estimate of $c_i$:

\begin{lemma}\label{prop:irg_supercrit}
	For any supercritical $\kappa$,
	 for any node $i$ satisfying $c_i < c_*$ and for any $K =
         k(n)$, where  $k:\mathbb{N} \to \mathbb{N}$ is an arbitrary
         positive  function satisfying $\lim_{n\to\infty}k(n)=\infty$,
         there exist a positive function $f_{\kappa}:\mathbb{N}\to \real$, such that $\lim_{n\to \infty}f_{\kappa}(n) = 0$ and
	 $$\frac{c_* - c_i}{n} \le (\musup_*(k(n)) - \musup_i(k(n)))\frac{\EE{C_1}}{n}   + f_{\kappa}(n).$$  
\end{lemma}
\begin{proof}

Define a kernel $\bar\kappa(x,y) = (1- \rho(y))\kappa(x,y) $, where $\rho$ is defined in Section~\ref{sec:branching}. 
By Theorem 6.7 in \cite{Bollobas:2007:PTI:1276871.1276872}, the
branching process $\GW_{\kappa}$ conditional on extinction is
subcritical and has the same distribution as the branching process
with parameters $\GW_{\bar\kappa}$. Then, by Lemma \ref{subcritical},
\begin{equation}
  \label{eq:dual}
  \PP{  B(i) > K |  B(i) < \infty } \le e^{- \lambda(\bar\kappa)
    k(n)}g(\bar\kappa)~.
\end{equation}  

  We relate the size of the connected component to the total progeny of branching process.  
 Following the stochastic dominance $C_i  \preceq B(i)$,
  \begin{align*}
  \musup_{i}(k(n)) = \PP{ |C_i| > k(n) } \le    \rho_i +  \PP{ B(i) > k(n) | B(i) < \infty } .
  \end{align*} 	
 This implies, for $n > n_0(\bar{\kappa})$,
	$$ \musup_i(k(n))\EE{|C_1|} - c_i <  \PP{ B(i) > k(n) | B(i) < \infty }\EE{|C_1|} + \rho_i\EE{|C_1|}  - c_i \le  e^{- \lambda(\bar\kappa) k(n)}g(\bar\kappa) \EE{|C_1|} + o_n(n). $$	
Finally, using that $\rho_* \le \musup_{*}(k(n))$, we get
\begin{align*}
\frac{c_* - c_i }{n}\le (\musup_*(k(n)) - \musup_i(k(n)))\frac{\EE{|C_1|} }{n}  + e^{- \lambda k(n) }g(\bar\kappa) \frac{\EE{|C_1|} }{n} + o_n(1) =  \delta^{sup}_i(k(n))\EE{|C_1|}   +  f_{\kappa}(n). 
\end{align*} 	
	
\end{proof}

\paragraph{Proof of Theorem~\ref{supercrit_reg_unknown_l}.}
	
	First, by (\ref{conditional_regret}), 
	\begin{align*}
	R^{\alpha}_T \le 
	\Delta_{\alpha,\max} + \EEcc{ \sum_{i\in V_0}  \Delta_{\alpha,i}   \EE{N_{i,T}} }{\mathcal{E}}~.
	\end{align*}
	
	As we mentioned before, with high probability, $C_2(G_n) = o_n(n)$,  which means that if $A_t \notin C_1(G_t)$, then $|C_{A_t}(G_t)| = o_n(n)$.  
	Since \irg is supercritical, $\argmax_{a} \mu_a = \argmax_{a} \rho_a$.
	Then, we can approximate distribution of rewards of arm $a$ by
        a Bernoulli distribution with parameter $\rho_a$. 
	Using the result of Proposition~\ref{prop:irg_supercrit}, we reduce the initial problem to the analysis of a multi-armed problem with arms $Z_1, \dots, Z_{|V_0|}$, where $Z_i \sim Ber(\musub_i)$, for $p_i$ defined in Proposition~\ref{prop:irg_supercrit}. 
	
	By Hoeffding’s inequality,
	$$\PPcc{A_{t+1} = i}{N_{i,t} \ge \frac{4 \log t}{(\delta^{sup}_{\alpha,i}(K)^2}} \le \frac{4}{t^2}.$$
	Then 
	\begin{align*}
	\EE{N_{i,T}} \le  \frac{4 \log T}{(\delta^{sup}_{\alpha,i}(k(n)))^2} + \sum_{t=|V_0|}^{T} \PPcc{A_{t+1} = i}{N_{i,t} \ge \frac{4 \log t}{(\delta^{sup}_{\alpha,i}(k(n)))^2}} \le \frac{4\log T}{(\delta^{sup}_{\alpha,i}(k(n)))^2} + 8.
	\end{align*}
		
	Now, observing that $\delta^{sup}_{\alpha,i}(k(n)) \le \max_{j
          \in V_0} \musup_j - \musup_i$ holds under the event $\mathcal{E}$, we obtain
	\begin{equation}
			R^{\alpha}_T \le   \Delta_{\alpha,\max} +  \EEcc{ \sum_{i\in V_0}  \Delta_{\alpha,i}  \pa{\frac{4 \log T}{(\delta^{sup}_{\alpha,i}(k(n)))^2} + 8 }}{\mathcal{E}},
	\end{equation}
	thus proving the first statement.
			
Now we fix an arbitrary $\varepsilon > 0$, we
split the set $V_0$ into two subsets: $U(\varepsilon) = \ev{a\in V_0: \delta^{sub}_{\alpha,i}(k(n))  \le \varepsilon}$ and $W(\varepsilon) = V_0 
\setminus U(\varepsilon)$, where we use the choice $\varepsilon = 2 \sqrt{|V_0| \EE{C_1}  \log T / T}$. 
Lemma~\ref{prop:irg_supercrit}  shows that $\frac{c_* - c_i}{n} \le	\frac{\musup_*(k(n)) - \musup_i(k(n))}{n}\EE{|C_1|} +f_{\kappa}(n)$.  Then there exists $n_0(\kappa)$, such that for any  \irg with $n > n_0(\kappa)$, $f_{\kappa}(n) \le \varepsilon$ holds. 
  Then, under the event $\mathcal{E}$, we have

\begin{align*}
	\frac{1}{n}\sum_{i\in V_0}  \Delta_{\alpha,i}   \EE{N_{i,T}} &= \sum_{i\in U(\varepsilon)}  \frac{\Delta_{\alpha,i}}{n}   \EE{N_{i,T}} +
	\sum_{i\in W(\varepsilon)}  \frac{\Delta_{\alpha,i}  }{n} \EE{N_{i,T}}\\
	&\le \pa{\frac{\varepsilon \EE{|C_1| }}{n} + \varepsilon }\sum_{i\in U(\varepsilon)}   \EE{N_{i,T}} +
	\sum_{i\in W(\varepsilon)}  \frac{\Delta_{\alpha,i}  }{n} \EE{N_{i,T}}
	\\
	&\le\pa{\frac{\varepsilon \EE{|C_1| }}{n} + \varepsilon } |V_0|T +   \sum_{i\in W(\varepsilon)}  \delta^{sub}_{\alpha,i}(k(n)) \frac{\EE{C_1}}{n}  \pa{
		\frac{4 \log T}{(\delta^{sub}_{\alpha,i}(k(n)))^2} }
	\\
	&\qquad +\varepsilon \sum_{i\in W(\varepsilon)}   \pa{\frac{4 \log T}{(\delta^{sub}_{\alpha,i}(k(n)))^2} }
	\\
	&\le  \pa{\frac{\varepsilon \EE{|C_1| }}{n} + \varepsilon } |V_0|T + |V_0| \pa{\frac{ \EE{|C_1| }}{n} +1 }  \frac{8\log T}{\varepsilon n } \\ 
	&\le 9  \pa{\frac{ \EE{|C_1| }}{n} + 1 } |V_0|\sqrt{T \log T}~,
\end{align*}
where the last step uses the choice $\varepsilon =  \sqrt{  \log T / T}$. Plugging in the choice of $|V_0|$ concludes 
the proof.
\qed

	\section{Degree observations.}\label{sec:degree}
\subsection{Subcritical case}
\label{sec:subcrit}

 Our main technical result is proving that 
nodes with the largest expected degrees $\mu^*$ are exactly 
the ones with the largest influence $c^*$, in both the stochastic block model and the Chung--Lu model, across both the subcritical and supercritical regimes. 
The following lemma states this result for the subcritical case.

\begin{lemma}\label{prop:subcr}
	Suppose that
	\begin{enumerate}
		\item $G$ is generated from a subcritical \sbm satisfying Assumption~\ref{ass:1}, or
		\item $G$ is generated from a subcritical \chlu. 
	\end{enumerate}
	Then, for any $i$ satisfying $\mu_i < \mu^*$, we have $c^* - c_i \le 2 c^*\pa{\mu^* - \mu_i} + O(1/n)$. 
\end{lemma}

Before stating and proving the lemma, we introduce some useful technical tools.
Since we suppose that \irg is subcritical, we have $\PP{B(i) = \infty} = 0$ and $x_i = \EE{B(i)}$ is finite.
First observe that the vector $x$ of expected total progenies satisfies the system of linear equations 
\[
x =  e + \frac{1}{n}A x~,
\] 
where $e$ is the vector with $e_i = 1$ for all $i$. 

For the analysis of the stochastic block model we define the vector $b \in \real^S$ with coordinates  $b_l = \mu_l$, $l=1,\ldots,S$, where by $ \mu_l$ we define the expected degree of the node from community $H_l$. Also we define vector $x' \in \real^S$ with coordinates $x'_l = \EE{B(l)}$, $l=1,\ldots,S$, where by $ B(l)$ we define the total progeny of the individual of type $l$. We define $x^* = \max_{i \in [n]} x_i$. 
Armed with this notation, we begin the proof Lemma~\ref{prop:subcr}, which consists of the following steps:
\begin{itemize}
	\item proving that for any $i,j \in V$, $x_i - x_j \le 2 x^* \pa{\mu_i - \mu_j}$, (Lemma~\ref{mean_order_sbm}, \ref{mean_order_chlu}),
	\item proving that  for any $i,j \in V$, $c_i - c_j = x_i - x_j + O(1/n)$ (Lemmas~\ref{mean_dominate}, \ref{mean_difference}).
\end{itemize}
These facts together lead to Lemma~\ref{prop:subcr}, given that $n$ is large enough to suppress the effects of the residual 
terms. We begin with analysing the relation between $b_l$ and $x'_l$ in a straightforward way:
\begin{lemma}[Coordinate order for mean of the total progeny in the SBM]
	\label{mean_order_sbm}
	Assume that \sbm is subcritical and that $K_{m\ell} = k>0$ holds for all $m\neq \ell$. 
	If two coordinates of $b$ are such that $b_l > b_m$, then we have $x'_l > x'_m$, and $x'_l - x'_m \le 2 x^* \pa{b_l - b_m}$.
\end{lemma}
\begin{proof}
	For the stochastic block model  with $S$ blocks, the system of equations $x = e + A x$ can be equivalently written as  $x' = e + M x'$, for $M = K \mbox{diag}(\alpha)\in\real^{S\times S}$, and $x'\in\real^S$, with $x'_m$ now standing for the expected total progeny associated with any 
	node of type $m$. Similarly, we define $b'_m$ as the expected degree of any node of type $m$.
	Notice that the system of equations $x' = e + M x'$ satisfied by $x'$ can be rewritten as $(I - M)x' = e$, where $I$ is the $S\times S$ 
	identity matrix. By exploiting our assumption on the matrix $K$ and defining $\gamma_m = K_{m,m} - k$,  this can be further rewritten as
	$$
	\left(\begin{pmatrix}
	1 - \alpha_{1}\gamma_1 & & \\
	& \ddots & \\
	& & 1 - \alpha_{S} \gamma_S
	\end{pmatrix} - 
	k  \begin{pmatrix}
	\alpha_{1} & \alpha_{2} & \cdots & \alpha_{S} \\
	\alpha_{1} & \alpha_{2} & \cdots & \alpha_{S} \\
	\vdots  & \vdots  & \ddots & \vdots  \\
	\alpha_{1} & \alpha_{2} & \cdots & \alpha_{S} 
	\end{pmatrix}\right) x'  = e,
	$$
	which means that for any $m$, $x_m'$ satisfies
	$$
	x_m' = \frac{1+k(\alpha\transpose x')}{1 - \alpha_m \gamma_m}.
	$$
	Also observe that $$b_m' = k(\alpha^T \bar{1}) + \alpha_m\gamma_m,$$
	so, for any pair of types $m$ and $\ell$, we have 
	$$
	x'_m - x'_\ell  = \frac{(1+k(\alpha\transpose x'))(\alpha_m\gamma_m - 
		\alpha_\ell\gamma_\ell)}{(1 - \alpha_m\gamma_m)(1 - \alpha_\ell\gamma_\ell)},
	$$
	which proves the first statement.
	
	To prove the second statement, observe that for any pair $\ell$ and $m$ of communities, we have either $\alpha_m \le \frac 12$ or 
	$\alpha_\ell \le \frac 12$ (otherwise we would have $\alpha_m + \alpha_\ell > 1$). To proceed, let $\ell$ and $m$ be such that $x'_m 
	\ge x'_\ell$, and let us study the case $\alpha_\ell \le \frac 12$ first. Here, we get
	\[
	\begin{split}
	x'_m - x'_\ell &= \frac{(1+k(\alpha\transpose x'))(\alpha_m\gamma_m - \alpha_\ell\gamma_\ell)}{(1 - \alpha_m\gamma_m)(1 - 
		\alpha_\ell\gamma_\ell)} = \frac{(\alpha_m\gamma_m - \alpha_\ell\gamma_\ell)}{(1 - \alpha_\ell\gamma_\ell)} x'_m
	\\
	&\le \frac{(\alpha_m\gamma_m - \alpha_\ell\gamma_\ell)}{(1 - \gamma_\ell/2)} x'_m \le 2 x'_m (b_m' - b'_\ell).
	\end{split}
	\]
	In the other case where $\alpha_m \le \frac 12$, we can similarly obtain
	\[
	x_m' - x_\ell' \le 2 x'_\ell (b_m' - b'_\ell) \le 2 x'_m (b_m' - b'_\ell).
	\]
	This concludes the proof.
\end{proof}
For the analysis of the \chunglu model, we define $\mu \in
\real^n$ as the vector of mean degrees. Then we may prove the following.
\begin{lemma}[Coordinate order for mean of the total progeny in the \chunglu model]
	\label{mean_order_chlu}
	Assume that \chlu is subcritical. If two nodes  are such that $\mu_i > \mu_j$, then we have $x_i > x_j$ and $x_i - 
	x_j \le x^* (\mu_i - \mu_j)$.
\end{lemma}
\begin{proof}
	From the system of equations $x = e + \frac{1}{n} A x$, the coordinates $x_i$ have the form 
	$$
	x_i = 1 + \frac{1}{n}\cdot w_i\pa{\sum_{j = 1}^{n}w_j x_j},
	$$
	which implies that $w_i \ge w_j$ holds if and only if $x_i \ge x_j$. This observation implies for $x^* = \max_i x_i$
	$$x_i - x_j \le \frac{1}{n}\cdot (w_i - w_j)\pa{\sum_{j = 1}^{n}w_j} x^* = \pa{\mu_i - \mu_j} x^*,$$
	thus concluding the proof.
\end{proof}
The next two lemmas establish the relationship between the expected component size $c_i$ of vertex $i$ and the expected
total progeny $x_i$ of the multi-type branching process seeded at vertex $i$.
\begin{lemma}
	\label{mean_dominate}
	For any $i$, the mean of the connected component associated with type $i$ is bounded by the mean of the total progeny: $c_i \le x_i$. 
\end{lemma}
\begin{proof}
	Now fix an arbitrary $i\in[n]$ and let $Y_{i,1}, Y_{i,2}, \dots, Y_{i,n}$ be independent Bernoulli random 
	variables with respective parameters $(A_{i,1}/n,A_{i,2}/n, \dots ,  A_{i,i}/n, \dots, A_{i,n}/n)$. 
	Consider a multitype binomial branching process where the individual of type $i$ produces  $Y_{i,j}$ individuals of type $j$, and 
	let $B_{\ber}(i)$ denote its total progeny when started from an individual of type $i$. 
	Recalling the Poisson branching process defined in Appendix~\ref{sec:branching} with offspring-distributions $X_{i,j}$, we can show 
	$B_{\ber}(i)\preceq B(i)$ using the relation $Y_{i,j} \preceq X_{i,j}$.

	Considering a node $a$ of type $i$, we can use Theorem 4.2 of \citet*{hofstad_2016} to bound the size of the the connected component 
	$C_a$ as $|C_a| \preceq B_{\ber}(i)$, which implies by transitivity of $\preceq$ that $|C_{a_i}| \preceq B(i)$. The proof is concluded by 
	appealing to Theorem~2.15 of \cite{hofstad_2016} that shows that stochastic domination implies an ordering of the means.
\end{proof}
Next we upper bound the excess that appears in the domination by the branching process:
\begin{lemma}
	\label{mean_difference}
	$x_i - c_i = O(\frac{1}{n})$~. 
\end{lemma} 

\begin{proof}
	As in Lemma~\ref{mean_dominate},  $B_{\ber}(i)$ denotes the total
	progeny of a Bernoulli branching process whose set of parameters corresponds to \irg. Then we may decompose the difference as $$x_i - c_i = x_i - \EE{B_{\ber}(i)} + \EE{B_{\ber}(i)} - c_i. $$ 
	
	Denote the set of edges in the connected component $C_a$ as $E(C_a)$ and the set of edges containing a vertex $v$ as $E(v)$.
	We call $|\Sw|$ the \emph{surplus}, which is the number of edges to
	be deleted from $E(C_a)$ such that the graph $C_a$ becomes a tree. Then, we have $\EE{B_{Ber}(i)} - c_i \le \EE{|\Sw|}$.  The expectation of the surplus may be
	written as
	\begin{align*}
	\EE{|\Sw|} =\EE{ \sum_{e\in E(C_a)} \mathbb{I} \{e \in \Sw \} } 
	= \sum_{k=1}^{\infty} \PP{|C_a|  = k }   \sum_{e\in E(C_a)} \EEcc{\mathbb{I}\{e\in \Sw \}}{ |C_a| = k } \\=  \frac{1}{2} \sum_{v \in C_a}  \sum_{e \in E(v)} \EEcc{\mathbb{I}\{e \in \Sw\}}{|C_a| = k} .
	\end{align*}

	Define $A_{\max}=\max_{i,j} A_{i,j}$ as the maximal element of the matrix $A$. 
	Then for an arbitrary vertex, the probability of an edge $e \in E(v)$ being in the surplus can be upper bounded as
	$$ \sum_{e \in E(v)} \EEcc{\mathbb{I}\{e \in \Sw\}}{ |C_a| = k } \le \frac{A_{\max} k}{ n}~. $$
	Then we may upper bound the sum as
	$$ \frac{1}{2} \sum_{v \in C_a}  \sum_{e \in E(v)} \EEcc{\mathbb{I}\{e \in \Sw\}}{|C_a| = k} \le  \frac{A_{\max} k^2}{n}~. $$
	Using our expression for $\EE{| \Sw |}$, we get
	\begin{align*}
	\EE{ | \Sw|} \le \sum_{k=1}^{\infty} \PP{|C_a|  = k }  \frac{A_{\max} k^2}{n} = \frac{A_{\max} \mathbb{E} 
		|C_a|^2}{ n}~.
	\end{align*}
	
	Now we notice that, by Le Cam's theorem, the  total variation
	distance between the sum of independent Bernoulli random variables with parameters $(A_{i, 1}/n, \dots, A_{i, n/n})$ and  the Poisson distribution $\poi(\sum_{j=1}^{n} A_{i,j}/n)$ is at most $2(\sum_{j=1}^{n} A^2_{i,j})/n$. 
	Using this fact and that the moments of the total progeny of a subcritical branching process do not scale with $n$ (cf.~Theorem~1 of \citealp*{Hua12}), we have $x_i - \EE{B_{Ber}(i)} = O\left(\frac{1}{n}\right)$, thus proving the lemma.
\end{proof}

	\subsection{Supercritical case}
\label{sec:supercrit}
\begin{lemma}\label{prop:supercrit}
	Suppose that
	\begin{enumerate}
		\item $G$ is generated from a supercritical \sbm satisfying Assumptions~\ref{ass:1} and~\ref{ass:2}, or
		\item $G$ is generated from a supercritical \chlu.
	\end{enumerate}
	Then, for any node $i$ satisfying $\mu_i < \mu^*$, we have $c^* - c_i \le c^*\pa{\mu^* - \mu_i} + o_n(n)$. 
\end{lemma}
The proof of Lemma~\ref{prop:supercrit} follows from the following lemmas for the stochastic block model and the \chunglu model and from the following relation between $c_i$ and $\rho_i$:
$$c_i = \rho_i\EE{|C_1|}+o_n(n)~,
$$
see \citet*[Chapter 9]{Bollobas:2007:PTI:1276871.1276872}.

\begin{lemma}[Coordinate order preserving in the  stochastic block model.]
	\label{coordinate_order_sbm}
	Assume the conditions of Lemma~\ref{prop:supercrit} and let $l_* = \argmax_l b_l$.
	Let $a\in \real^S$ be any vector such that $a_l \in [0, a_{l_*}]$ for all $l$. Then $\pa{\Phi_M(a)}_{l_*} \ge \pa{\Phi_M(a)}_{l}$.
\end{lemma}
\begin{proof}
	Let us fix two arbitrary indices $l$ and $l'$. By the definition of $\Phi_M$, we have
	\begin{align*}
	\pa{\Phi_M(a)}_l &= 1 - e^{-((\sum_{m \neq l} \alpha_m a_m)k + \alpha_l K_{l,l} a_l)}~,\\
	\pa{\Phi_M(a)}_{l'} &= 1 - e^{-((\sum_{m \neq l'} \alpha_m a_m)k + \alpha_{l'} K_{l',l'} a_{l'})}~.
	\end{align*}
	Notice that if $l$ and $l'$  satisfy
	$$\pa{\sum_{m \neq l} \alpha_m a_m}k + \alpha_l K_{l,l} a_l \ge 
	\pa{\sum_{m \neq l'} \alpha_m a_m}k + \alpha_{l'} 
	K_{l',l'} a_{l'},$$
	we have $\pa{\Phi_M(a)}_l \ge \pa{\Phi_M(a)}_{l'}$.
	Now, using the facts that
	\begin{itemize}
		\item $\sum_{m \neq l} \alpha_m a_m - \sum_{m \neq l'} \alpha_m a_m =  
		\alpha_{l'} a_{l'}  - \alpha_l a_l$,
		\item $\alpha_l K_{l,l} \ge\alpha_{l} k$, 
		\item $\alpha_l K_{l,l} +\alpha_{l'} k \ge \alpha_{l'} k_{l', l'} + \alpha_{l} k$ and 
		\item $a_{l} - a_{l'} \ge 0$,
	\end{itemize}
	we can verify that
	\begin{eqnarray*}
		\lefteqn{
			\alpha_l K_{l,l} a_l + \alpha_{l'} k a_{l'} - \alpha_{l} k a_l -
			\alpha_{l'} K_{l', l'} a_{l'}  } \\
		& = & (\alpha_l K_{l,l} + \alpha_{l'} k) 
		a_{l'} 
		+ (a_{l} - a_{l'})\alpha_l K_{l,l} - (\alpha_{l'} K_{l', l'} +
		\alpha_{i} k)a_{l'} - (a_{l} - a_{l'})\alpha_{l} k \ge 0,
	\end{eqnarray*}
	thus proving the lemma.
\end{proof}
\begin{lemma}[Order of coordinates of eigenvector in the SBM]
	\label{eigenvector_order_sbm}
	Let $a$ be the eigenvector corresponding to the largest eigenvalue
	$\lambda$ of the matrix $M = K \mbox{diag}(\alpha)$. Then if  $l_* = \argmax_l b_l$, we have 
	$a_{l_*} \ge a_l$ for $l \neq l_*$.
\end{lemma}
\begin{proof}
	If $a$ is an eigenvector of $M$, then for coordinates $l, l'$:
	
	$$\begin{cases} \pa{\sum_{m \neq l} \alpha_m a_m}k + \alpha_l k_{l,l} a_l = \lambda a_l, \\
	\pa{\sum_{m \neq m'} \alpha_m a_m}k + \alpha_l K_{l',l'} a_{l'} = \lambda a_{l'}  \end{cases}$$
	By the Perron--Frobenius theorem and our conditions on matrix $M$, $\lambda$ is a real number larger than one.
	Denote $C = k \sum_{m \neq l, m \neq l'}  \alpha_m a_m$, $x = a_l$, $y = a_{l'}$, $a =  \alpha_l K_{l,l}$, $b = \alpha_{l'} k$, $c = 
	\alpha_{l} k$, $d = \alpha_{l'} k_{l', l'}$. Then,
	\begin{equation}\label{silly_sys}
	\begin{cases} C+ ax +by = \lambda x, \\
	C + cx + dy = \lambda y  \end{cases}
	\end{equation}
	Let $r = 1 + \epsilon$ be such that $y = rx = (1 + \epsilon)x$. Then
	$$\begin{cases} \frac{C}{x}+ a +b + b\epsilon = \lambda,  \\
	\frac{C}{x} + c + d + d\epsilon = \lambda + \lambda\epsilon  \end{cases}$$
	and therefore
	$$\frac{C}{x} + c + d + d\epsilon = \frac{C}{x}+ a +b + b\epsilon + \lambda\epsilon~. $$
	Rearranging the terms and using the fact that $a+b \ge c + d$, we have
	$$
	0 \le (a+b) - (c+d) = (d-b-\lambda)\epsilon~.
	$$ 
	Since $K_{l,l} \ge k$, we have $\alpha_l k_{l,l} \ge \alpha_l  k$ and $a \ge c$.
	
	We consider two cases separately: First, if $b \ge d$, we have $d - b - \lambda < 0$, which implies $\epsilon < 0$ and $y < x$, therefore 
	proving $a_l > a_{l'}$ for this case. In the case when $b < d$, we have
	$a+b \ge c + d$ and $\frac{d-b}{a-c} \le 1$. Subtracting the two equalities of the linear system \ref{silly_sys}, we get
	$$\lambda(1-r) = (a-c)\pa{1 - \frac{d-b}{a-c} r}~.$$ 
	Now, since  $\frac{d-b}{a-c} \le 1$, we have $\lambda \ge a - c$, which implies $\lambda \ge d -b$ and $d - b - \lambda \le 0$, thus leading to
	$\epsilon \le 0$ and  $y \le x$, therefore proving  $a_l \ge a_{l'}$ for this case. 
\end{proof}

\begin{lemma}[Order of coordinates of eigenvector in the \chunglu model]
	\label{eigenvector_order_chlu}
	Let $a$ be the eigenvector corresponding to the largest eigenvalue
	$\lambda$ of the matrix $A$. Then if  $i_* = \argmax_m b_m$, we have 
	$a_{i_*} \ge a_j$ for $j \neq i_*$.
\end{lemma}
\begin{proof}
	It is easy to see that the only eigenvector of $A$ corresponding to a non-zero eigenvalue is $a=w$ with $\lambda_{max} = 
	w\transpose w/n$: 
	$$\frac{1}{n}A w = \frac{1}{n}\cdot (ww\transpose)w = \frac{w\transpose w}{n}\cdot w.$$ 
	The proof is concluded by observing that the maximum coordinate of the vector 
	$b$ corresponds to the maximum coordinate of $w$, due to the equality
	\[b_i = \frac{1}{n}\cdot w_i \sum_{j=1}^n w_j.\]
\end{proof}

\begin{lemma}[Coordinate order preserving in the \chunglu model]
	\label{coordinate_order_chlu}
	Assume the conditions of Lemma~\ref{prop:supercrit} and let $i_* = \argmax_i b_i$.
	Let $a=(a_1, \dots, a_n)$ be such that $a_j \in [0, a_{i_*}]$ for all $j$. Then $\pa{\Phi_A(a)}_{i_*} \ge \pa{\Phi_A(a)}_{j}$.
\end{lemma}
\begin{proof}
	Let us fix two arbitrary indices $i$ and $i'$. By the definition of $\Phi_A$, we have
	\begin{align*}
	\pa{\Phi_A(a)}_i &= 1 - e^{-w_i(\sum_{j = 1}^n w_j a_j)}~.
	\end{align*}
	Then, using the fact that $w = a$, we have $\pa{\Phi_A(a)}_{i_*} \ge \pa{\Phi_A(a)}_{j}$,
	thus proving the lemma.
\end{proof}

We finally study the maximal fixed point of the operator $\Phi_A$, keeping in mind this fixed point is exactly the survival-probability 
vector $\rho$ of the multi-type Galton--Watson branching process \citet*{Bollobas:2007:PTI:1276871.1276872}. By Lemma~5.9 of 
\citet*{Bollobas:2007:PTI:1276871.1276872}, this is the unique fixed point satisfying $\rho_i > 0$ for all $i$. The following lemma shows 
that $\rho_i$ takes its maximum at $i_* = \argmax_i b_i$, concluding the proof of Lemma~\ref{prop:supercrit}.
\begin{lemma}[Fixed point coordinate domination]
	\label{lma_mean}
	Let $\rho$ be the unique non-zero fixed point of $\Phi_A$, and let $i_* = \argmax_i b_i$. Then, 
	$\rho_{i_*} \ge \rho_j$ and $\rho_{i_*} - \rho_j \le \rho^* \pa{b_{i_*} - b_j}$ holds for all $j\neq i_*$.
\end{lemma} 
\begin{proof}
	Letting $a$ be the eigenvector of $A$ that corresponds to the largest eigenvalue $\lambda$, Lemma~\ref{eigenvector_order_chlu} and \ref{eigenvector_order_sbm} 
	guarantee $a_{i_*} \ge a_j$ for  $j\neq i^*$. 
	Let $\epsilon > 0$ be such that $\epsilon \le \frac{1 - 1/\lambda}{a^*}$, where $a^* = \max_{i=1,\ldots,S}a_i$. Then by Lemma~5.13 
	of \citet*{Bollobas:2007:PTI:1276871.1276872}, $\Phi_M(\epsilon a) \ge \epsilon a$ holds elementwise for the two vectors. 
	
	Since the coordinates of the vector $\epsilon a$ are positive, we can appeal to Lemma 5.12
	of \citet*{Bollobas:2007:PTI:1276871.1276872} to show that iterative application of $\Phi_A$ converges to the fixed point $\rho$:
	letting $\Phi_A^{m}$ be the operator obtained by iterative application of $\Phi_A$ for $m$ times, we have
	$\lim_{m\ra \infty} \Phi_A^{m}(\epsilon a) = \rho$, where $\rho$ satisfies $\rho \ge \epsilon a \ge 0$ and $\Phi_A (\rho) = \rho > 
	0$.
	By Lemmas~\ref{eigenvector_order_chlu} and \ref{eigenvector_order_sbm}  we have $\rho_{i_*} \ge \rho_j$, for $i_* \neq j$ 
	for both the SBM and the Chung--Lu models, proving the first statement.
	
	The second statement can now be proven directly as
	\begin{eqnarray*}
		\lefteqn{
			\rho_{i_*} - \rho_i = e^{-(\frac{1}{n}A\rho)_{j}} - e^{-(\frac{1}{n}A\rho)_{i_*}} = e^{- \frac{1}{n}\sum_{j}^{n} A_{i_* j}\rho_j} - e^{- \frac{1}{n}\sum_{j}^{n} A_{ij}\rho_j}   
		} \\
		& = & e^{- \frac{1}{n}\sum_{j}^{n} A_{i_* j}\rho_j} (1 - e^{- \frac{1}{n}\sum_{j}^{n} A_{ij}\rho_j- A_{i_* j}\rho_j}  ) \le 
		e^{- \frac{1}{n} \sum_{j}^{n} A_{i_* j}\rho_j}  \left( \frac{1}{n} \sum_{j}^{n} (A_{i_* j} - A_{i j})\rho_{i_*}  \right) 
		\\
		& \le & \rho^* (b_{i_*} - b_{i}),
	\end{eqnarray*}
	where the first inequality uses the relation $1-e^{-z} \le z$ that holds for all $z\in\real$, and the last step uses the fact that 
	$A\rho$ has positive elements.
\end{proof}

	\subsection{Proofs of Theorems~\ref{thm:reg_knowT1}, ~\ref{thm:doubling} and~\ref{thm:kronecker} .}\label{app:klucb}

Having 
established that, in order to minimize regret in our setting, it is sufficient to design an algorithm that quickly identifies the 
nodes with the highest degree. It remains to show that our algorithms indeed achieve this goal. We do this below by providing a bound on the 
expected number of times $\EE{N_{T,i}} = \EE{\sum_{t=1}^T \II{A_t =
    i}}$ that the algorithm picks a suboptimal node $i$ such that $c_i <
c^*$, and then using this guarantee to bound the regret.

Without loss of generality, we assume that $V_0 = \ev{1,2,\dots,|V_0|}$.  The key to our regret bounds is the following guarantee on the 
number of suboptimal actions taken by \ducbv. 
\begin{theorem}[Number of suboptimal node plays in \ducb] \label{thm:n_klucb}
	Define $\eta_i = \pa{\max_{j\in V_0} \mu_j- \mu_i} / 3$. 
	The number of times that any  node $i \in \ev{i: \mu_{i} < \max_{j\in V_0} \mu_j}$ is chosen by \ducbv
	satisfies
	\begin{equation}
	\mathbb{E} N_{T,i} \le  \frac{\mu^* \pa{2 + 6 \log T}}{\eta_{i}^2} + 3~.
	\end{equation}
\end{theorem}
The proof is largely based on the analysis of the kl-UCB algorithm due to \citet*{cappe:hal-00738209}, with some additional tools borrowed 
from \citet*{2017arXiv170207211M}, crucially using that the degree distribution of each node is stochastically dominated by an appropriately 
chosen Poisson distribution. Specifically, letting $Z_i$ be a Poisson random variable with mean $\EE{X_{t,i}}$, we have $\EE{e^{sX_{t,i}}} 
\le \EE{e^{sZ_i}}$ for all $s$. It turns out that this property is sufficient for the \klucb analysis to go through in our case, which is an 
observation that may be of independent interest. 

Before delving into the proof, we introduce some useful notation. We start by defining $Y_{i,1}, \dots, Y_{i,n}$ as independent  Bernoulli  random  variables with 
respective parameters $\mathbb{B} = (A_{i,1}/n,$
	$A_{i,2}/n, \dots , A_{i,n}/n)$, and noticing that the degree $X_{t,i}$ can be written as a sum 
$X_i = 
\sum_{j\neq i} Y_{i,j}$. The following lemma, used several times in our proofs, relates this quantity to a Poisson distribution with the 
same mean.
\begin{lemma}\label{lem:domination}
	Let $i\in[S]$ and let $Y_{i,1}, Y_{i,2}, \dots, Y_{i,n}$ be independent Bernoulli random variables with respective parameters $p_{i,1}, 
	p_{i,2}, \dots , p_{i,n}$, and let $Z_i$ be a Poisson random 
	variable with parameter $\mu_i = \sum_{j \neq i} p_{i,j} $. Defining $X_i = \sum_{j\neq i} Y_{i,j}  $, we have $\EE{e^{sX_i}} \le 
	\EE{e^{sZ_i}}$ for all $s\in\real$.
\end{lemma}
\begin{proof}
	Fix an arbitrary $s\in\real$ and $i\in[n]$. By direct calculations, we obtain
	\begin{align*}
	\mathbb{E} e^{s X_i} 
	&= \prod_{j=1}^{n} \left(\mathbb{E} e^{s Y_{i,j}}\right) \le \prod_{j=1}^{n} \left(1 + p_{i,j}(e^s - 1)\right)
	\le  \prod_{j=1}^{n}  \exp\pa{p_{i,j} \cdot  (e^s - 1)} ,
	\end{align*}
	where the last step follows from the elementary inequality $1+x \le e^{x}$ that holds for all $x\in\real$. The proof is concluded by 
	observing that $\mathbb{E} e^{s Z_i} = \exp\pa{\mu\pa{e^{s}-1}}$ and using the definition of $\mu$.
\end{proof}
For simplicity, we also introduce the notation $\psi_{\mathbb{B} }(s) = \log \EE{e^{sX}}$ and
$\phi_{\lambda}(s) = \log \mathbb{E} e^{s Z_i} = \lambda (e^s - 1)$. 
The proof below repeatedly refers to the Fenchel conjugate of 
$\phi_\lambda$ defined as
\begin{align*}
\phi_{\lambda}^*(z) = \sup_{s \in \real} \{sz - \phi(s)\} = z \log \left(\frac{z}{\lambda}\right) + \lambda - z 
\end{align*}
for all $z\in\real$. Finally, we define
$d(\mu, \mu') = \mu' - \mu + \mu \log\left(\frac{\mu}{\mu'}\right)$ for all $\mu,\mu' > 0$, noting that $\phi_{\lambda}^*(z) = 
d(z,\lambda)$.
\paragraph{Proof of Theorem~\ref{thm:n_klucb}.}
The statement is proven in four steps. Within this proof, we refer to nodes as \emph{arms} and use 
$K$ to denote the size of 
$V_0$. We use the notation $f(t) = 3\log t$.
\paragraph{Step 1.} We begin by rewriting the expected number of draws $\EE{N_i}$ for any suboptimal arm $i$ as 
\[
\mathbb{E}N_i = \EE{\sum_{t = K}^{T-1} \mathbb{I} \{A_{t+1} = i  \}} = \sum_{t = K}^{T-1} \mathbb{P} \{A_{t+1} = i  \}.
\]
By definition of our algorithm, at rounds $t > K$, we have $A_{t+1} = i$ only if $U_{i} >   U_{i^*i}$. This leads to the 
decomposition:
\begin{align*}
\{A_{t+1} = a \} &\subseteq \{\mu^* \ge U_{i^*}(t)  \} \cup \{\mu^* < U_{i^*}(t) \text{ and } A_{t+1} = a \}
\\
&\subseteq \{\mu^* \ge U_{i^*}(t) 
\} \cup \{\mu^* < U_{i}(t) \text{ and } A_{t+1} = a \} 
\end{align*}

Steps~2 and~3 are devoted to bounding the probability of the two events above.

\paragraph{Step 2.} Here we aim to upper bound
\begin{equation}\label{eq:underest}
\sum_{t = K}^{T-1} \PP{\mu^* \ge U_{i^*}(t)}.
\end{equation}
Note, that $\left\{ U_{i^*}(t) \le \mu^*\right\} = \left\{\hat{\mu}_{i^*}(t) \le U_{i^*}(t) \le \mu^*\right\}.$
Since $d(\mu, \mu') = \mu' - \mu + \mu\log(\frac{\mu}{\mu'})$ is non-decreasing in its second argument on $[\mu, + \infty)$, and by 
definition of  $U_{i^*}  = \sup \{ \mu: d( \hat{\mu}_{i^*}(t), \mu) \le \frac{f(t)}{N_{i^*(t)}} \}  $ we have
\[
\left\{\mu^* \ge U_{i^*}(t) \right\} \subseteq \left\{ \hat{\mu}_{i^*}(t) \le U_{i^*}(t) \le \mu^* \text{ and }  d(\hat{\mu}_{i^*}(t), 
\mu^*) \ge 
\frac{f(t)}{N_{i^*}(t)} \right\},
\]
Taking a union bound over the possible values of $N_{i^*}(t)$ yields 
\begin{align*}\label{eq:arm-decomposition}
\left\{\mu^* \ge U_{i^*}(t) \right\} \subseteq  \bigcup_{n=1}^{t - K + 1} \left\{ \mu^* \ge \hat{\mu}_{i^*, n} \text{ and } 
d(\hat{\mu}_{i^*, n}, \mu^*) \ge \frac{f(t)}{n} \right\} =  \bigcup_{n=1}^{t - K + 1} D_n(t),
\end{align*}
where the event $D_n(t)$ is defined through the last step.
Since $d(\mu, \mu^*)$ is decreasing and continuous in its first argument on  $[0, \mu^*)$, either $d(\hat{\mu}_{i^*, n}, \mu^*) < 
\frac{f(t)}{n}$ on this interval and $D_n(t)$ is the empty set, or there exists a unique $z_n \in [0, \mu^*)$ such that $d(z_n, \mu^*) =  
\frac{f(t)}{n}$. Thus, we have
\[
\bigcup_{n=1}^{t - K + 1} D_n(t) \subseteq  \bigcup_{n=1}^{t - K + 1} \left\{ \hat{\mu}_{i^*, n} \le  z_n \right\}.
\]
For $\lambda < 0$, let us define $\psi(\lambda)$ as the cumulant-generating function of the sum of binomials with 
parameters $\mathbb{B} $, and let $\phi(\lambda)$ be the cumulant-generating function of a Poisson random variable with parameter 
$\mu^*$. With this notation, we have for \emph{any} $\lambda < 0$ that
\begin{align*}
\PP{ \hat{\mu}_{i^*, n} \le  z_n } &=   \PP{ \exp(\lambda \hat{\mu}_{i^*, n}) \ge  \exp(\lambda  z_n) } 
\\
&= \PP{ \exp\pa{\lambda \sum_{i = 1}^n X_{i^*, i} - n\psi(\lambda)} \ge  \exp(n \lambda  z_n - n\psi(\lambda))} 
\\
&\le \left(\frac{\mathbb{E} e^{\lambda X_{i^*,1}}}{e^{\psi(\lambda)}}\right)^n e^{-n(\lambda z_n - \psi(\lambda))}  
\le  e^{-n(\lambda z_n - \psi(\lambda))},
\end{align*}
where the last step uses the definition of $\psi(\lambda)$. Now fixing $\lambda^* = \argmax_{\lambda} \{ \lambda z_n - \phi(\lambda) \}  = 
\log(z_n/\mu^*) < 0$, we get by Lemma \ref{lem:domination} that
\begin{align*}
e^{-n(\lambda^* z_n - \psi(\lambda^*))} \le e^{-n(\lambda^* z_n - \phi(\lambda^*))}  = e^{-n \phi^*_{\mu^*}(z_n)} = e^{ - n d(z_n, 
	\mu^*)}~. 
\end{align*}
In view of the definition of $z_n$ and $f(t)$, this gives the bound
\begin{align*}
e^{ - n d(z_n, \mu^*)}  =  e^{- f(t)} = \frac{1}{t^3},
\end{align*}
which leads to
\begin{align*}
\sum_{t = K}^{T-1} \PP{\mu^* \ge U_{i^*}(t)} \le  \sum_{t = K}^{T-1}  \sum_{n=1}^{t - K + 1} \frac{1}{t^3} < 2,
\end{align*}
thus concluding this step.

\paragraph{Step 3.} 
In this step, we borrow some ideas by \citet[Proof of Theorem~2, step~2]{2017arXiv170207211M}  to 
upper bound the sum
\begin{equation} \label{sum_3rd}
B = \sum_{t = K}^{T-1} \PP{\mu^* < U_{i}(t) \text{ and } A_{t+1} = i }.
\end{equation}
Writing $\eta = \eta_i = \{\mu^* - \mu_i\}/3$ for ease of notation, we have
\begin{align*} 
\{\mu^* < U_{i}(t) \text{ and } A_{t+1} = i \} &\subseteq \ev{\mu^* - \eta < U_{i}(t) \text{ and } A_{t+1} = i} 
\\
&\subseteq 
\ev{d(\hat{\mu}_{i}(t), \mu^* - \eta ) \le f(t)/N_i(t) \text{ and } A_{t+1} = i}.
\end{align*}
Thus, we have
\begin{align*} 
B &\le \sum_{t = K}^{T-1} \PP{d(\hat{\mu}_{i}(t), \mu^* - \eta ) \le f(t)/N_i(t) \text{ and } A_{t+1} = i}
\\
&\le \sum_{n = 1}^{T} \PP{ d(\hat{\mu}_{i,n}, \mu^* - \eta ) \le f(T)/n }
\end{align*}
Defining the integer $n(\eta)$ as
$$
n(\eta) = \left \lceil \frac{f(T)}{d(\mu_{i} + \eta, \mu^* - \eta)} \right \rceil,
$$
we have $f(T)/n \le d(\mu_{i} + \eta, \mu^* - \eta)$ for all $n \ge
n(\eta) $. Thus, we may further upper bound $B$ as
\begin{align*}
B &\le n(\eta) - 1 + \sum_{n = n(\eta)}^{T} \PP{d(\hat{\mu}_{i,n}, \mu^* - \eta ) \le f(T)/n} 
\\
&\le \frac{f(T)}{d(\mu_{i} + \eta, \mu^* - \eta)}  +   \sum_{n = n(\eta)}^{T} \PP{d(\hat{\mu}_{i,n}, \mu^* - \eta ) \le d(\mu_{i} + 
	\eta, \mu^* - \eta)}. 
\end{align*}
By definition of $\eta$, we have
$$\ev{\hat{\mu}_{i,n}, \mu^* - \eta ) \le d(\mu_{i} + \eta, \mu^* - \eta)} \subseteq \ev{\hat{\mu}_{i,n} \ge \mu_i + \eta},$$
which implies
\begin{align*}
\sum_{n = n(\eta)}^{T}  \PP{d(\hat{\mu}_{i,n}, \mu^* - \eta ) \le d(\mu_{i} + \eta, \mu^* - \eta)} \le \sum_{n = n(\eta)}^{T} 
\PP{\hat{\mu}_{i,n} \ge \mu_i + \eta}.
\end{align*}
By an argument analogous to the one used in the previous step, we get for a well-chosen $\lambda$ that
\begin{align*}
&\sum_{n = n(\eta)}^{T} \PP{\hat{\mu}_{i,n} \ge \mu_i + \eta} 
\le \PP{\exp(\lambda \hat{\mu}_{i,n}) \ge \exp(\lambda(\mu_i+\eta))} 
\\
&\qquad\qquad\qquad= \sum_{n = n(\eta)}^{T} \PP{\exp(\lambda\sum_{j = 1}^n X_{i, j} - n\psi(\lambda)) \ge \exp(n\lambda(\mu_i + \eta) 
	- n\psi(\lambda)) } \\
&\qquad\qquad\qquad\le \sum_{n = n(\eta)}^{T} \pa{\frac{\EE{e^{\lambda X_{i,j}}}}{e^{\psi(\lambda)}}}^n e^{-n(\lambda(\mu_i+\eta) - 
	\psi(\lambda))} 
\\
&\qquad\qquad\qquad
\le \sum_{n = n(\eta)}^{T} e^{-n(\lambda(\mu_i+\eta) - \phi(\lambda))} 
= \sum_{n = n(\eta)}^{T} e^{-n d(\mu_i+\eta, \mu_i)} 
\\
&\qquad\qquad\qquad\le 
\sum_{n = n(\eta)}^{ \infty}  e^{-n d(\mu_i+\eta, 
	\mu_i)} \le \frac{1}{e^{d(\mu_i+\eta, \mu_i)} - 1} \le \frac{1}{d(\mu_i+\eta, \mu_i)},
\end{align*}
where the last step uses the elementary inequality $1+x \le e^{x}$ that holds for all $x\in \real$. 
\paragraph{Step 4.}
Putting together the results from the first three steps, we get
\begin{align*}
\mathbb{E} N_{i} \le  3  +  \frac{1}{d(\mu_i+\eta, \mu_i)} +  \frac{3 \log T}{d(\mu_{i} + \eta, \mu^* - \eta)}.
\end{align*}
We conclude by taking a second-order Taylor-expansion of $d(\mu_i + \eta,\mu_i)$ in $\eta$ to obtain for some  $\eta' \in [0,\eta]$
that
$$
d(\mu_i+\eta, \mu_i) = \frac{\eta^2}{2(\mu_i+\eta')} \ge \frac{\eta^2}{2(\mu_i+\eta)}.
$$
Taking into account the definition of $\eta$, we get
$$\frac{1}{d(\mu_i+\eta, \mu_i)} \le  \frac{2\mu^*}{\eta^2}.$$
An identical argument can be used to bound $\pa{d(\mu_{i} + \eta, \mu^* - \eta)}^{-1} \le 2\mu^*/\eta^2$.
\qed

The remainder of the section 
uses 
Theorem~\ref{thm:n_klucb} to prove  Theorem~\ref{thm:reg_knowT1}. The proof of Theorem~\ref{thm:doubling} follows from 
similar ideas and some additional technical arguments.

\paragraph{Proof of Theorem~\ref{thm:reg_knowT1}.}
	
	First, by (\ref{conditional_regret}),
	\begin{align*}
	R^{\alpha}_T
	\le 
	\Delta_{\alpha,\max} + \EEcc{ \sum_{i\in V_0}  \Delta_{\alpha,i}   \EE{N_{T,i}} }{\mathcal{E}}~.
	\end{align*}
	Now, observing that $\delta_{\alpha,i} \le 3\eta_i$ holds under event $\mathcal{E}$, we appeal to Theorem \ref{thm:n_klucb} to obtain
	\begin{equation}\label{conditional_regret2}
	R^{\alpha}_T \le   \Delta_{\alpha,\max} +  \EEcc{ \sum_{i\in V_0}  \Delta_{\alpha,i}  \pa{
			\frac{\mu^*\pa{18 + 27\log T}}{\delta_{i,\alpha}^2} + 3} }{\mathcal{E}},
	\end{equation}
	thus proving the first statement.
	
	Next, we turn to proving the second statement regarding worst-case guarantees. To do this, we appeal to 
	Propositions~\ref{prop:subcr} and~\ref{prop:supercrit} that respectively show $\Delta_i \le 2 c^* \delta_i + O(1/n)$ and $\Delta_i \le c^* 
	\delta_i + o(n)$ for the sub- and supercritical settings, and we use our assumption that $n$ is 
	large enough so that we have $\Delta_i \le 3 c^* \delta_i$ in both settings. 
	Specifically, we observe that $\delta_i = \Theta_n(1)$ by our sparsity assumption and $c^*$ is $\Theta_n(1)$ in the subcritical and 
	$\Theta_n(n)$ supercritical settings, so, for large enough $n$, the superfluous $O(1/n)$ and $o(n)$ terms can be 
	respectively bounded by $c^* \delta_i$. 
	To proceed, let us fix an arbitrary $\varepsilon > 0$ and
	split the set $V_0$ into two subsets: $U(\varepsilon) = \ev{i\in V_0: \delta_{\alpha,i} \le \varepsilon}$ and $W(\varepsilon) = V_0 
	\setminus U(\varepsilon)$. Then, under event $\mathcal{E}$, we have
	\begin{align*}
	\sum_{i\in V_0}  \Delta_{\alpha,i}   \EE{N_{T,i}} &= \sum_{i\in U(\varepsilon)}  \Delta_{\alpha,i}   \EE{N_{T,i}} +
	\sum_{i\in W(\varepsilon)}  \Delta_{\alpha,i}   \EE{N_{T,i}}
	\\
	&\le 3c^* \varepsilon \sum_{i\in U(\varepsilon)} \EE{N_{T,i}} + 3c^* \sum_{i\in W(\varepsilon)}  \delta_{\alpha,i} \pa{
		\frac{\mu^*\pa{18 + 27\log T}}{\delta_{\alpha,i}^2}} + 3 |W(\varepsilon)|\Delta_{\alpha,\max} 
	\tag{by Theorem~\ref{thm:n_klucb}} 
	\\
	&\le 3c^* \varepsilon T + 3c^* \sum_{i\in W(\varepsilon)}  \frac{\mu^*\pa{18 + 27\log T}}{\delta_{\alpha,i}} + 3 |V_0| \Delta_{\alpha,\max}
	\\
	&\le 3c^* \pa{\varepsilon T + |V_0| \frac{\mu^*\pa{18 + 27\log T}}{\varepsilon}} + 3 |V_0| \Delta_{\alpha,\max}
	\\
	&\le 6c^* \sqrt{T |V_0| \mu^*\pa{18 + 27\log T}} + 3 |V_0| \Delta_{\alpha,\max},
	\end{align*}
	where the last step uses the choice $\varepsilon = \sqrt{|V_0| \mu^*\pa{18 + 27\log T} / T}$. Plugging in the choice of $|V_0|$ concludes 
	the proof.
\qed

\paragraph{Proof of Theorem~\ref{thm:doubling}.}
We start by assuming that $\alpha < 1/2$.
Also notice that for a uniformly sampled set of nodes $U$, the probability of $U$ not containing a vertex from $V^*_{\alpha} $ is 
bounded as
\[
\PP{U\cap V^*_{\alpha}  = \emptyset} \le (1 - \alpha)^{|U|}.
\]
By the definition of $V_k$, this gives that the probability of not having sampled a  node from $V^*_{\alpha} $ in period $k$ of the 
algorithm is bounded as
\[
\PP{V_k \cap V^*_{\alpha} =\emptyset} \le (1 - \alpha)^{|V_k|} \le \beta^{-k}.
\]

For each period $k$, the expected regret can bounded as the weighted sum of two terms: the expected regret of \ducb$(V_k)$ in period $k$ 
whenever $V_k\cap V^*_{\alpha}$ is not empty, and the trivial bound $\Delta_{\alpha,\max} \beta^k$ in the complementary case. Using the 
above 
bound on the 
probability of this event and appealing to Theorem~\ref{thm:n_klucb} to bound the regret of \ducb$(V_k)$, we can bound the expected regret 
as
\begin{align*}
\EE{R^{\alpha}_T} &\le \sum_{k=1}^{k_{\max}} \pa{ \beta^k \frac{1}{\beta^k} \Delta_{\alpha, \max} +  \sum_{i \in V_k} \Delta_{\alpha,i} 
	\pa{ \frac{\mu^*\pa{2+3\log \beta^k}}{\delta_{\alpha,i}^2} +  3 }} 
\\
&\le k_{\max}  \Delta_{\alpha, \max}  + \sum_{k=1}^{k_{\max}}  \pa{  \sum_{i \in V_k}  \Delta_{\alpha,i} 
	\pa{ \frac{\mu^*\pa{2+3k \log \beta}}{\delta_{\alpha,i}^2} +  3 }} 
\\
&\le k_{\max}  \Delta_{\alpha, \max} +     \sum_{i \in \overline{V}}  \Delta_{\alpha,i}  \pa{  \pa{3 +  
		\frac{2\mu^*}{\delta_{\alpha,i}^2} } (k_{\max} + 1) +  \frac{3 \log \beta (k_{\max}+1)^2 }{ 2\delta_{\alpha,i}^2}}~.
\end{align*}
The proof of the first statement is concluded by upper bounding the number of restarts up to time $T$ as
$k_{\max} \le \frac{\log T}{\log \beta}$.

The second statement is proven by an argument analogous to the one used in the proof of Theorem~\ref{thm:reg_knowT1}, and straightforward 
calculations.
\qed

\paragraph{Proof of Theorem~\ref{thm:kronecker}.}

	For a node $i$, such that $s_i$ contains $l_i$ ones, the expected degree is   \[\mu_i = (\zeta+\beta)^{l_i}(\beta+\gamma)^{k-l_i}.\] Since $\zeta>\gamma>\beta $, we get that $\mu_i>\mu_j$ if $l_i>l_j$. By
	symmetry of the nodes in the Kronecker graph, if two nodes $i$ and $j$ are such that $l_i = l_j$, then $c_i = c_j$. This implies that for any node $i$, $c_i$ is a function of $l_i$. Then we may choose nodes $i$ and $j$ such that $s_i \ge s_j$ coordinate-wise. Then, using the condition $\zeta>\gamma>\beta $, it is straightforward to see that for any vertex $k$, the probability of the edge $(i,k)$ is greater
	than that of edge $(j,k)$.
	This implies that the connected component is a monotone function of the degree.  
	
	Theorem 9.10 in \cite{49013}  shows that for a graph \krgraph there exists $b(P)$ such that a
	subgraph of \krgraph induced by the vertices $i \in H$ of weight $l_i\ge k/2$ is connected
	with probability at least $1-n^{-b(P)}$. 
	We denote by  $\mathcal{H}$ the event that  the subgraph of \krgraph induced by the vertices of $H$ of weight $l\ge k/2$ is connected. This implies that under event $\mathcal{H}$, 

	$|C_{i} | = |\max_j C_{j}|$ for all $i \in H$ and $|C_{i} |\le |\max_j C_{j}|$ for all $i \notin H$. Then we get
	\begin{align*}
		c^*_{\alpha}-c_i &= \EEcc{\max_{j}|C_j|}{\mathcal{H}}\PP{\mathcal{H}} + \EEcc{|C^*_{\alpha}|}{\mathcal{H}^c}\PP{\mathcal{H}^c}-\EEcc{\max_{j}|C_j|}{\mathcal{H}}\PP{\mathcal{H}} - \EEcc{|C_i|}{\mathcal{H}^c}\PP{\mathcal{H}^c}\\
		& \le \EEcc{|C_*|}{\mathcal{H}^c}\PP{\mathcal{H}^c}\le n^{1-b(P)}.
	\end{align*} 

  For all $i\in V_0\setminus H$, $\delta_{\alpha,i} \ge \pa{(\zeta+\beta)(\beta+\gamma)}^{k/2} - (\zeta+\beta)^{k/2-1}(\beta+\gamma)^{k/2+1} = \pa{(\zeta+\beta)(\beta+\gamma)}^{k/2}\pa{1-\frac{\beta+\gamma}{\zeta+\beta}} $. Since we consider the regime, where $(\zeta+\beta)(\beta+\gamma)>1$, we get that $\delta_{\alpha, i} > \pa{1-\frac{\beta+\gamma}{\zeta+\beta}}$. For all $i,j \in H$, $\delta_{\alpha,i} = (\zeta + \beta)^{l_i}(\beta+\gamma)^{k-l_i} \ge (\zeta+\beta)^{k/2-1}(\beta+\gamma)^{k/2} (\zeta - \gamma)\ge (\zeta - \gamma)/(\zeta+\beta) $. In the same way as we analysed the regret of the stochastic block model and Chung--Lu model, we can write 
  $$R_{T}^ { \alpha} \le nT\PP{\mathcal{E}^c} + \EEcc{\sum_{i \in V_0}\Delta_{\alpha,i} \EE{N_{i,T}}}{\mathcal{E}}.$$
  Applying Theorem~\ref{thm:n_klucb}, we get 
\begin{align*}
	\frac{R_{T}^{ \alpha}}{n} &\le   \EEcc{ \sum_{i\in  V_0 \setminus H } \frac{\Delta_{\alpha,i}}{n}  \pa{
	\frac{\mu^*(2 + 6\log T)}{\pa{1-\frac{\beta+\gamma}{\zeta+\beta}}^2} + 3}  }{\mathcal{E}}\\
&  +n^{-b(P)} \left\lceil\frac{\log (nT)}{\log(2)}\right\rceil \pa{
	\frac{\mu^*(2 + 6\log T)(\zeta+\beta)^2}{(\zeta - \gamma)^2} + 3} + 1~.
\end{align*}
Applying $| V_0 \setminus H|\le |V_0|$ and $\Delta_{\alpha,i} \le n$, we get the final bound on the regret. 
\qed

	\bibliography{influence}

\begin{thebibliography}{40}
\providecommand{\natexlab}[1]{#1}
\providecommand{\url}[1]{\texttt{#1}}
\expandafter\ifx\csname urlstyle\endcsname\relax
  \providecommand{\doi}[1]{doi: #1}\else
  \providecommand{\doi}{doi: \begingroup \urlstyle{rm}\Url}\fi

\bibitem[Agarwal et~al.(2010)Agarwal, Bartlett, and Dama]{pmlr-v9-agarwal10a}
Alekh Agarwal, Peter Bartlett, and Max Dama.
\newblock Optimal allocation strategies for the dark pool problem.
\newblock In Yee~Whye Teh and Mike Titterington, editors, \emph{Proceedings of
  the Thirteenth International Conference on Artificial Intelligence and
  Statistics}, volume~9 of \emph{Proceedings of Machine Learning Research},
  pages 9--16, Chia Laguna Resort, Sardinia, Italy, 13--15 May 2010. PMLR.
\newblock URL \url{http://proceedings.mlr.press/v9/agarwal10a.html}.

\bibitem[Auer et~al.(2002{\natexlab{a}})Auer, Cesa-Bianchi, and
  Fischer]{auer2002finite}
Peter Auer, Nicol\`{o} Cesa-Bianchi, and Paul Fischer.
\newblock Finite-time analysis of the multiarmed bandit problem.
\newblock \emph{Machine Learning}, 47\penalty0 (2-3):\penalty0 235--256,
  2002{\natexlab{a}}.

\bibitem[Auer et~al.(2002{\natexlab{b}})Auer, Cesa-Bianchi, Freund, and
  Schapire]{auer2002bandit}
Peter Auer, Nicol\`{o} Cesa-Bianchi, Yoav Freund, and Robert~E. Schapire.
\newblock The nonstochastic multiarmed bandit problem.
\newblock \emph{SIAM J. Comput.}, 32\penalty0 (1):\penalty0 48--77,
  2002{\natexlab{b}}.
\newblock ISSN 0097-5397.

\bibitem[Bart\'ok et~al.(2012)Bart\'ok, Zolghadr, and
  Szepesv\'ari]{bartok2012adaptive}
G\'abor Bart\'ok, Navid Zolghadr, and Csaba Szepesv\'ari.
\newblock An adaptive algorithm for finite stochastic partial monitoring, 2012.

\bibitem[Bollob\'{a}s et~al.(2007)Bollob\'{a}s, Janson, and
  Riordan]{Bollobas:2007:PTI:1276871.1276872}
B{\'e}la Bollob\'{a}s, Svante Janson, and Oliver Riordan.
\newblock The phase transition in inhomogeneous random graphs.
\newblock \emph{Random Struct. Algorithms}, 31\penalty0 (1):\penalty0 3--122,
  August 2007.
\newblock ISSN 1042-9832.

\bibitem[Bubeck and Cesa-Bianchi(2012)]{bubeck12survey}
S\'ebastien Bubeck and Nicol\`o Cesa-Bianchi.
\newblock \emph{Regret Analysis of Stochastic and Nonstochastic Multi-armed
  Bandit Problems}.
\newblock Now Publishers Inc, 2012.

\bibitem[Capp{\'e} et~al.(2013)Capp{\'e}, Garivier, Maillard, Munos, and
  Stoltz]{cappe:hal-00738209}
Olivier Capp{\'e}, Aur{\'e}lien Garivier, Odalric-Ambrym Maillard, R{\'e}mi
  Munos, and Gilles Stoltz.
\newblock {Kullback-Leibler Upper Confidence Bounds for Optimal Sequential
  Allocation}.
\newblock \emph{{Annals of Statistics}}, 41\penalty0 (3):\penalty0 1516--1541,
  2013.

\bibitem[Carpentier and Valko(2016)]{CaVa16}
Alexandra Carpentier and Michal Valko.
\newblock Revealing graph bandits for maximizing local influence.
\newblock In \emph{Artificial Intelligence and Statistics}, pages 10--18, 2016.

\bibitem[Cesa-Bianchi and Lugosi(2006)]{CBL06}
Nicolo Cesa-Bianchi and Gabor Lugosi.
\newblock \emph{Prediction, Learning, and Games}.
\newblock Cambridge University Press, USA, 2006.
\newblock ISBN 0521841089.

\bibitem[Chaudhuri et~al.(2009)Chaudhuri, Freund, and Hsu]{CFH09}
Kamalika Chaudhuri, Yoav Freund, and Daniel~J Hsu.
\newblock A parameter-free hedging algorithm.
\newblock In \emph{Advances in neural information processing systems}, pages
  297--305, 2009.

\bibitem[Chen et~al.(2010)Chen, Wang, and Wang]{ChenSIM10}
Wei Chen, Chi Wang, and Yajun Wang.
\newblock Scalable influence maximization for prevalent viral marketing in
  large-scale social networks.
\newblock In \emph{Proceedings of the 16th ACM SIGKDD International Conference
  on Knowledge Discovery and Data Mining}, KDD '10, pages 1029--1038, New York,
  NY, USA, 2010. ACM.
\newblock ISBN 978-1-4503-0055-1.

\bibitem[Chen et~al.(2013{\natexlab{a}})Chen, Lakshmanan, and
  Castillo]{ChLaCa13}
Wei Chen, Laks~VS Lakshmanan, and Carlos Castillo.
\newblock Information and influence propagation in social networks.
\newblock \emph{Synthesis Lectures on Data Management}, 5\penalty0
  (4):\penalty0 1--177, 2013{\natexlab{a}}.

\bibitem[Chen et~al.(2013{\natexlab{b}})Chen, Wang, and Yuan]{CWY13}
Wei Chen, Yajun Wang, and Yang Yuan.
\newblock Combinatorial multi-armed bandit: General framework and applications.
\newblock In \emph{International Conference on Machine Learning}, pages
  151--159, 2013{\natexlab{b}}.

\bibitem[Chernov and Vovk(2010)]{CV10}
Alexey Chernov and Vladimir Vovk.
\newblock Prediction with advice of unknown number of experts.
\newblock In \emph{Proceedings of the Twenty-Sixth Conference on Uncertainty in
  Artificial Intelligence}, pages 117--125. AUAI Press, 2010.

\bibitem[Chung and Lu(2002)]{CL02}
Fan Chung and Linyuan Lu.
\newblock The average distances in random graphs with given expected degrees.
\newblock \emph{Proceedings of the National Academy of Sciences}, 99\penalty0
  (25):\penalty0 15879--15882, 2002.

\bibitem[Chung and Lu(2006)]{chung2006complex}
Fan Chung and Linyuan Lu.
\newblock \emph{Complex Graphs and Networks (Cbms Regional Conference Series in
  Mathematics)}.
\newblock American Mathematical Society, USA, 2006.
\newblock ISBN 0821836579.

\bibitem[Frieze and Karonski(2015)]{49013}
Alan Frieze and Michal Karonski.
\newblock \emph{Introduction to Random Graphs}.
\newblock Cambridge University Press, New York, 2015.
\newblock Hardcover.

\bibitem[Garivier and Capp{\'e}(2011)]{GaCa11}
Aur{\'e}lien Garivier and Olivier Capp{\'e}.
\newblock The {KL-UCB} algorithm for bounded stochastic bandits and beyond.
\newblock In \emph{Proceedings of the 24th annual Conference On Learning
  Theory}, pages 359--376, 2011.

\bibitem[{Huaming}(2012)]{Hua12}
W.~{Huaming}.
\newblock {On total progeny of multitype Galton-Watson process and the first
  passage time of random walk with bounded jumps}.
\newblock \emph{ArXiv e-prints}, September 2012.

\bibitem[Kakade et~al.(2009)Kakade, Kalai, and Ligett]{KKL09}
Sham~M Kakade, Adam~Tauman Kalai, and Katrina Ligett.
\newblock Playing games with approximation algorithms.
\newblock \emph{SIAM Journal on Computing}, 39\penalty0 (3):\penalty0
  1088--1106, 2009.

\bibitem[Kempe et~al.(2003)Kempe, Kleinberg, and Tardos]{KeKlTa03}
David Kempe, Jon Kleinberg, and {\'E}va Tardos.
\newblock Maximizing the spread of influence through a social network.
\newblock In \emph{Proceedings of the ninth ACM SIGKDD international conference
  on Knowledge discovery and data mining}, pages 137--146. ACM, 2003.

\bibitem[Khim et~al.(2019)Khim, Jog, and Loh]{khim2019adversarial}
Justin Khim, Varun Jog, and Po-Ling Loh.
\newblock Adversarial influence maximization.
\newblock In \emph{2019 IEEE International Symposium on Information Theory
  (ISIT)}, pages 1--5. IEEE, 2019.

\bibitem[Komiyama et~al.(2015)Komiyama, Honda, and Nakagawa]{2969239.2969439}
Junpei Komiyama, Junya Honda, and Hiroshi Nakagawa.
\newblock Regret lower bound and optimal algorithm in finite stochastic partial
  monitoring.
\newblock In \emph{Proceedings of the 28th International Conference on Neural
  Information Processing Systems - Volume 1}, NIPS’15, page 1792–1800,
  Cambridge, MA, USA, 2015. MIT Press.

\bibitem[Koolen and Van~Erven(2015)]{KvE15}
Wouter~M Koolen and Tim Van~Erven.
\newblock Second-order quantile methods for experts and combinatorial games.
\newblock In \emph{Conference on Learning Theory}, pages 1155--1175, 2015.

\bibitem[Kovalenko(1971)]{kovalenko1971theory}
IN~Kovalenko.
\newblock Theory of random graphs.
\newblock \emph{Cybernetics}, 7\penalty0 (4):\penalty0 575--579, 1971.

\bibitem[Lai and Robbins(1985)]{LR85}
TL~Lai and Herbert Robbins.
\newblock Asymptotically efficient adaptive allocation rules.
\newblock \emph{Advances in Applied Mathematics}, 6\penalty0 (1):\penalty0
  4--22, 1985.

\bibitem[Lai(1987)]{Lai87}
Tze~Leung Lai.
\newblock Adaptive treatment allocation and the multi-armed bandit problem.
\newblock \emph{The Annals of Statistics}, pages 1091--1114, 1987.

\bibitem[Lattimore and Szepesv{\'a}ri(2020)]{LSz20}
Tor Lattimore and Csaba Szepesv{\'a}ri.
\newblock \emph{Bandit algorithms}.
\newblock Cambridge University Press, 2020.

\bibitem[Leskovec et~al.(2010)Leskovec, Chakrabarti, Kleinberg, Faloutsos, and
  Ghahramani]{leskovec2010kronecker}
Jure Leskovec, Deepayan Chakrabarti, Jon Kleinberg, Christos Faloutsos, and
  Zoubin Ghahramani.
\newblock Kronecker graphs: an approach to modeling networks.
\newblock \emph{Journal of Machine Learning Research}, 11\penalty0 (2), 2010.

\bibitem[Leskovec(2008)]{leskovec2008dynamics}
Jurij Leskovec.
\newblock \emph{Dynamics of large networks}.
\newblock PhD thesis, Carnegie Mellon University, School of Computer Science,
  Machine Learning~…, 2008.

\bibitem[Leskovec et~al.(2005)Leskovec, Chakrabarti, Kleinberg, and
  Faloutsos]{10.1007/11564126_17}
Jurij Leskovec, Deepayan Chakrabarti, Jon Kleinberg, and Christos Faloutsos.
\newblock Realistic, mathematically tractable graph generation and evolution,
  using kronecker multiplication.
\newblock In Al{\'i}pio~M{\'a}rio Jorge, Lu{\'i}s Torgo, Pavel Brazdil, Rui
  Camacho, and Jo{\~a}o Gama, editors, \emph{Knowledge Discovery in Databases:
  PKDD 2005}, pages 133--145, Berlin, Heidelberg, 2005. Springer Berlin
  Heidelberg.
\newblock ISBN 978-3-540-31665-7.

\bibitem[Luo and Schapire(2014)]{LS14}
Haipeng Luo and Robert~E Schapire.
\newblock A drifting-games analysis for online learning and applications to
  boosting.
\newblock In \emph{Advances in Neural Information Processing Systems}, pages
  1368--1376, 2014.

\bibitem[Maillard et~al.(2011)Maillard, Munos, and Stoltz]{maillard11dmed}
Odalric-Ambrym Maillard, R{\'e}mi Munos, and Gilles Stoltz.
\newblock A finite-time analysis of multi-armed bandits problems with
  {Kullback--Leibler} divergences.
\newblock In \emph{Proceedings of the 24th Annual Conference on Learning
  Theory}, pages 497--514, 2011.

\bibitem[M\'enard and Garivier(2017)]{2017arXiv170207211M}
Pierre M\'enard and Aurélien Garivier.
\newblock A minimax and asymptotically optimal algorithm for stochastic
  bandits.
\newblock In \emph{Proceedings of the 28th International Conference on
  Algorithmic Learning Theory}, pages 223--237, 2017.

\bibitem[Perrault et~al.(2020)Perrault, Healey, Wen, and
  Valko]{perrault2020budgeted}
Pierre Perrault, Jennifer Healey, Zheng Wen, and Michal Valko.
\newblock Budgeted online influence maximization.
\newblock In \emph{ICML2020}, 2020.

\bibitem[Streeter and Golovin(2009)]{SG09}
Matthew Streeter and Daniel Golovin.
\newblock An online algorithm for maximizing submodular functions.
\newblock In \emph{Advances in Neural Information Processing Systems}, pages
  1577--1584, 2009.

\bibitem[van~der Hofstad(2016)]{hofstad_2016}
Remco van~der Hofstad.
\newblock \emph{Random Graphs and Complex Networks}, volume~1 of
  \emph{Cambridge Series in Statistical and Probabilistic Mathematics}.
\newblock Cambridge University Press, 2016.
\newblock \doi{10.1017/9781316779422}.

\bibitem[Vaswani et~al.(2015)Vaswani, Lakshmanan, and Schmidt]{VaLaSc15}
Sharan Vaswani, Laks Lakshmanan, and Mark Schmidt.
\newblock Influence maximization with bandits.
\newblock \emph{arXiv preprint arXiv:1503.00024}, 2015.

\bibitem[Wang and Chen(2017)]{WaCh17}
Qinshi Wang and Wei Chen.
\newblock Improving regret bounds for combinatorial semi-bandits with
  probabilistically triggered arms and its applications.
\newblock In \emph{Advances in Neural Information Processing Systems}, pages
  1161--1171, 2017.

\bibitem[Wen et~al.(2017)Wen, Kveton, Valko, and Vaswani]{WeKvVaVa17}
Zheng Wen, Branislav Kveton, Michal Valko, and Sharan Vaswani.
\newblock Online influence maximization under independent cascade model with
  semi-bandit feedback.
\newblock In \emph{Advances in Neural Information Processing Systems}, pages
  3026--3036, 2017.

\end{thebibliography}

\end{document}